\documentclass{article}
\usepackage{microtype}
\usepackage{graphicx}
\usepackage{booktabs} 

\usepackage[accepted]{icml2019}

\usepackage[utf8]{inputenc} 
\usepackage[T1]{fontenc}    
\usepackage{hyperref}       
\usepackage{etoolbox}
\makeatletter
\patchcmd\@combinedblfloats{\box\@outputbox}{\unvbox\@outputbox}{}{\errmessage{\noexpand patch failed}}
\makeatother
\usepackage{url}            
\usepackage{booktabs}       
\usepackage{amsfonts}       
\usepackage{nicefrac}       
\usepackage{microtype}      
\usepackage{bigints}
\usepackage{graphicx,amsfonts,amsmath,amssymb,amsthm,psfrag,color}
\usepackage{cancel}
\usepackage{natbib}

\usepackage{amssymb}
\usepackage{algorithm}
\usepackage{algorithmic}
\usepackage{mathrsfs}
\usepackage{subcaption}

\def\Dir{{\sf Dir}}

\def\intr{{\sf int}}

\def\ones{{\bf 1}}
\def\lspan{{\sf span}}

\newenvironment{blue}{\color{blue}}{}

\long\def\comment#1{}
\newcommand{\Scal}{\ensuremath{\mathcal{S}}}
\newcommand{\Bscr}{\ensuremath{\mathscr{B}}}
\newcommand{\Qdist}{\ensuremath{\mathbb{Q}}}
\newcommand{\Cov}{\ensuremath{\textrm{Cov}}}

\DeclareMathOperator*{\conv}{\textrm{Conv}}

\def\bbP{\mathbb P}
\def\reals{\mathbb R}
\def\Ex{\mathbb E}
\DeclareMathOperator*{\argmax}{argmax}
\DeclareMathOperator*{\argmin}{argmin}

\renewcommand{\algorithmicrequire}{\textbf{Input:}}
\renewcommand{\algorithmicensure}{\textbf{Output:}}

\theoremstyle{definition}
\newtheorem{definition}{Definition}
\newtheorem{theorem}{Theorem}

\newtheorem{lemma}{Lemma}

\icmltitlerunning{Dirichlet Simplex Nest and Geometric Inference}

\begin{document}

\twocolumn[
\icmltitle{Dirichlet Simplex Nest and Geometric Inference}



\icmlsetsymbol{equal}{*}

\begin{icmlauthorlist}
\icmlauthor{Mikhail Yurochkin}{ibm,mit-ibm,equal}
\icmlauthor{Aritra Guha}{umich,equal}
\icmlauthor{Yuekai Sun}{umich}
\icmlauthor{XuanLong Nguyen}{umich}
\end{icmlauthorlist}

\icmlaffiliation{ibm}{IBM Research, Cambridge}
\icmlaffiliation{mit-ibm}{MIT-IBM Watson AI Lab}
\icmlaffiliation{umich}{Department of Statistics, University of Michigan}

\icmlcorrespondingauthor{Mikhail Yurochkin}{mikhail.yurochkin@ibm.com}

\icmlkeywords{Admixtures, topic modeling, NMF}

\vskip 0.3in
]

\printAffiliationsAndNotice{\icmlEqualContribution}

\begin{abstract}
We propose Dirichlet Simplex Nest, a class of probabilistic models suitable for a variety of data types, and develop fast and provably accurate inference algorithms by accounting for the model's convex geometry and low dimensional simplicial structure. By exploiting the connection to Voronoi tessellation and properties of Dirichlet distribution, the proposed inference algorithm is shown to achieve consistency and strong error bound guarantees on a range of model settings and data distributions. The effectiveness of our model and the learning algorithm is demonstrated by simulations and by analyses of text and financial data.\footnote{Code:  \url{https://github.com/moonfolk/VLAD}}
\end{abstract}

\section{Introduction}
For many complex probabilistic models, especially those with latent variables, the probability distribution of interest can be represented as an element of a convex polytope in a suitable ambient space, for which model fitting may be cast as the problem of finding the extreme points of the polytope. For instance, a mixture density can be identified as a point in a convex set of distributions whose extreme points are the mixture components. In the well-known topic model \cite{blei2003latent} for text analysis, a document corresponds to a point drawn from the topic polytope, its extreme points are the topics to be inferred. This convex geometric viewpoint provides the basis for posterior contraction behavior analysis of topic models, as well as developing fast geometric inference algorithms~\cite{nguyen2015posterior,tang2014understanding,yurochkin2016geometric,yurochkin2017conic}.

The basic topic model -- the Latent Dirichlet Allocation (LDA) of \citet{blei2003latent}, as well as the comparable finite admixtures developed in population genetics \cite{pritchard2000inference} were originally designed for categorical data. However, there are many real world applications in which the convex geometric probabilistic modeling continues to be a sensible approach, even if observed measurements are no longer discrete-valued, but endowed with a variety of distributions. To expand the scope of admixture modeling for a variety of data types, we propose to study Dirichlet Simplex Nest (DSN), a class of probabilistic models that generalizes the LDA, and to develop fast and provably accurate inference algorithms by accounting for the model's convex geometry and its low dimensional simplicial structure. 

The generative process given by a DSN is simple to describe: starting from a simplex $\Bscr$ of $K$ vertices embedded in a high-dimensional ambient space $\Scal$, one draws random points from the $\Bscr$'s relative interior according to a Dirichlet distribution. Given each such point, a data point is generated according to a suitable probability kernel $F$. For the general simplex nest, $\Scal$ can be any vector space of dimensions $D \geq K-1$, while the probability kernel $F$ can be taken to be Gaussian, Multinomial, Poisson, etc, depending on the nature of the observed data (continuous, categorical or counts, resp.).
If $\Scal$ is standard probability simplex, and $F$ a Multinomial distribution over categories, then the model is reduced to the familiar LDA model of \citet{blei2003latent}.

Although several geometric aspects of the DSN can be found in a vast array of well-known models in the literature, they were rarely treated together. First, viewing data as noisy observations from the low-dimensional affine hull that contains $\Bscr$, our model shares an assumption that can be found in both classical factor analysis and non-negative matrix factorization (NMF) models~\cite{Lee2001Algorithms}, as well as the work of~\citet{anandkumar2012spectral,arora2012learning} arising in topic models. Second, the convex constraints (i.e., linear weights of a convex combination are non-negative and sum to one) are present in all latent variable probabilistic modeling, even though most dominant computational approaches to inference such as MCMC sampling \cite{griffiths2004finding} and variational inference \cite{blei2003latent,hoffman2013stochastic,kucukelbir2017automatic} do not appear to take advantage of the underlying convex geometry. 

As is the case with topic models, scalable parameter estimation is a key challenge for the Dirichlet Simplex Nest. Thus, our main contribution is a novel inference algorithm that accounts for the convex geometry and low dimensionality of the latent simplex structure endowed with a Dirichlet distribution. Starting with an original geometric technique of \citet{yurochkin2016geometric}, we present several new ideas allowing for more effective learning of asymmetric simplicial structures and the Dirichlet's concentration parameter for the general DSN model, thereby expanding its applicability to a broad range of data distributions. We also establish statistical consistency and estimation error bounds for the proposed algorithm.

The paper proceeds as follows. Section \ref{sec:simplex_nest} describes Dirichlet Simplex Nest models and reviews existing geometric inference techniques. Section \ref{sec:sn_estimation} elucidates the convex geometry of the DSN via its connection to the Voronoi Tessellation of simplices and the structure of Dirichlet distribution on low-dimensional simplices. This helps motivate the proposed Voronoi Latent Admixture (VLAD) algorithm. Theoretical analysis of VLAD is given in Section \ref{sec:theory}. Section \ref{sec:experiments} presents an exhaustive comparative study on simulated and real data. We conclude with a discussion in Section \ref{sec:discussion}.

\section{Dirichlet Simplex Nest}
\label{sec:simplex_nest}

\begin{figure*}[t]
\vskip -0.1in
\begin{subfigure}{.245\textwidth}
  \centering
  \includegraphics[width=\linewidth]{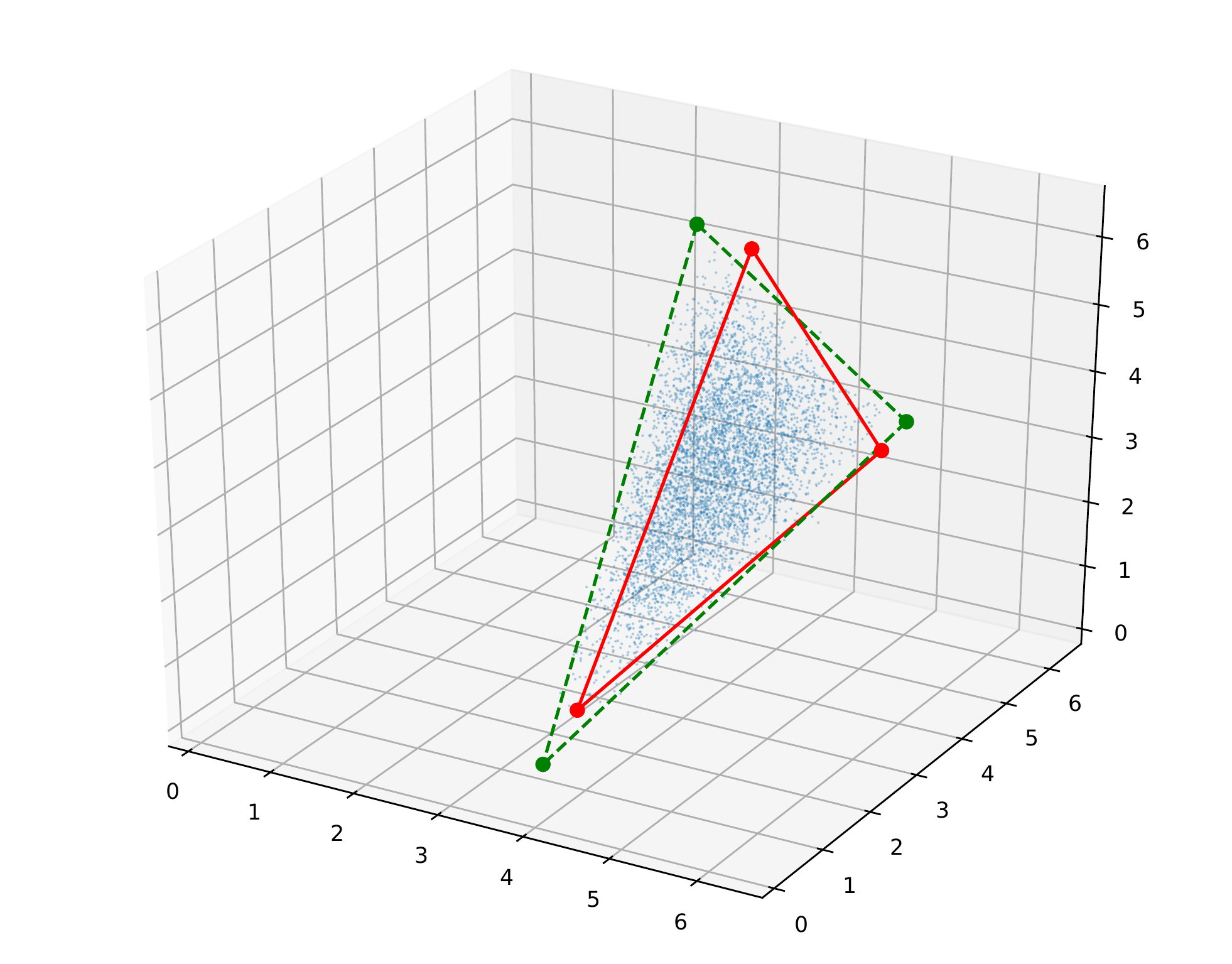}
  \caption{GDM; time $\approx$ 1s}
  \label{fig:gdm}
\end{subfigure}
\begin{subfigure}{.245\textwidth}
  \centering
  \includegraphics[width=\linewidth]{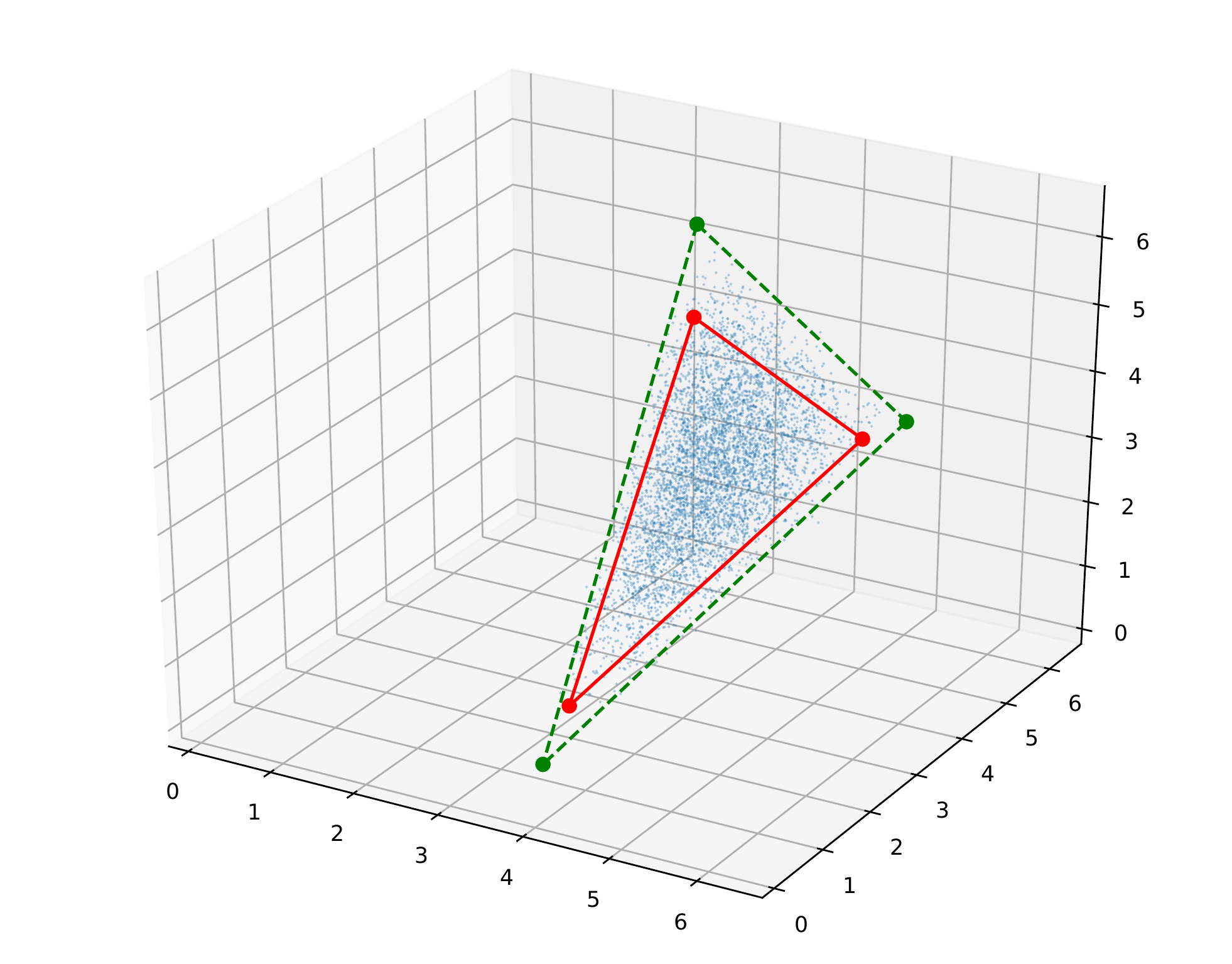}
  \caption{Xray; time < 1s}
  \label{fig:xray}
\end{subfigure}
\begin{subfigure}{.245\textwidth}
  \centering
  \includegraphics[width=\linewidth]{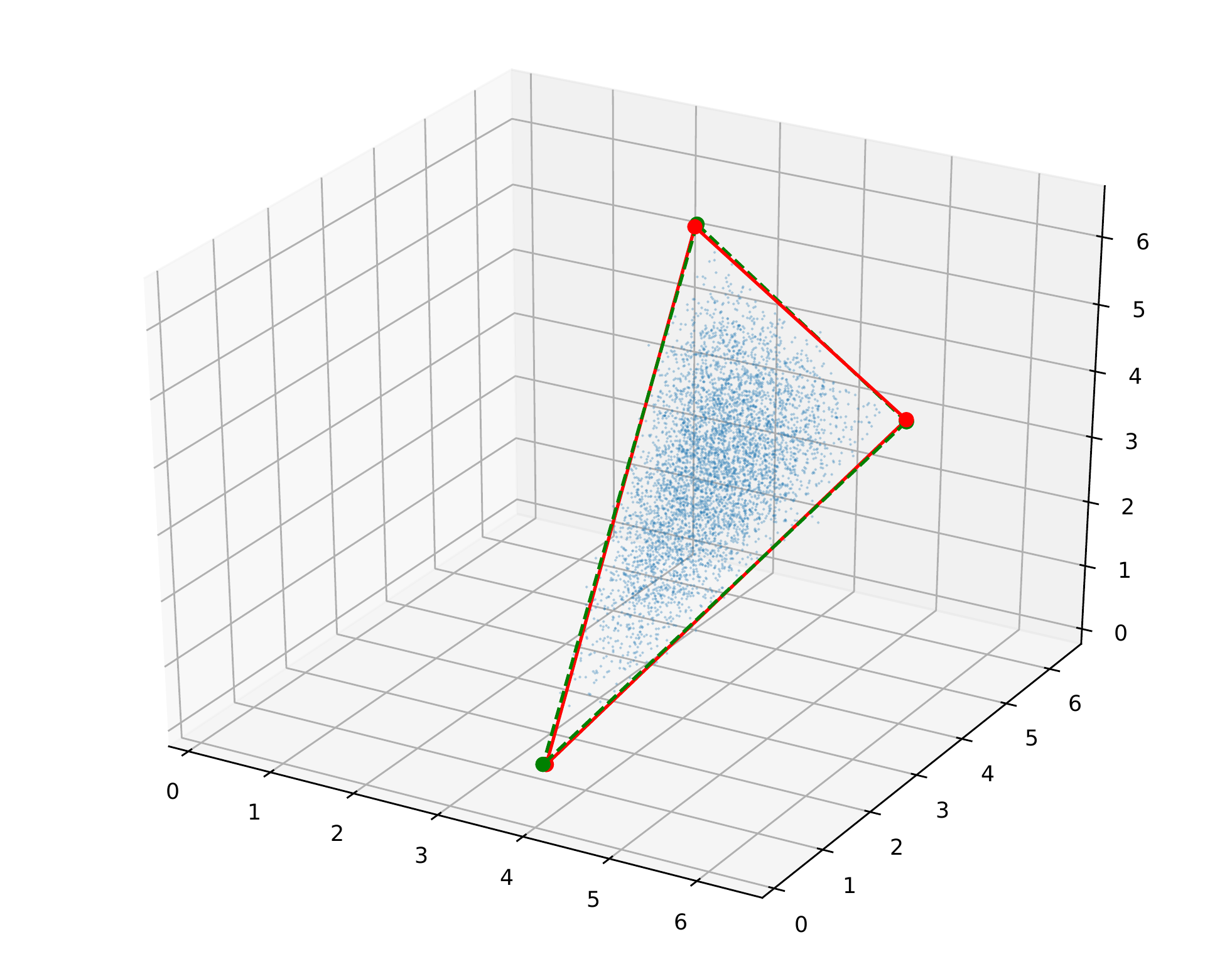}
  \caption{HMC; time $\approx$ 10m}
  \label{fig:stan-hmc}
\end{subfigure}
\begin{subfigure}{.245\textwidth}
  \centering
  \includegraphics[width=\linewidth]{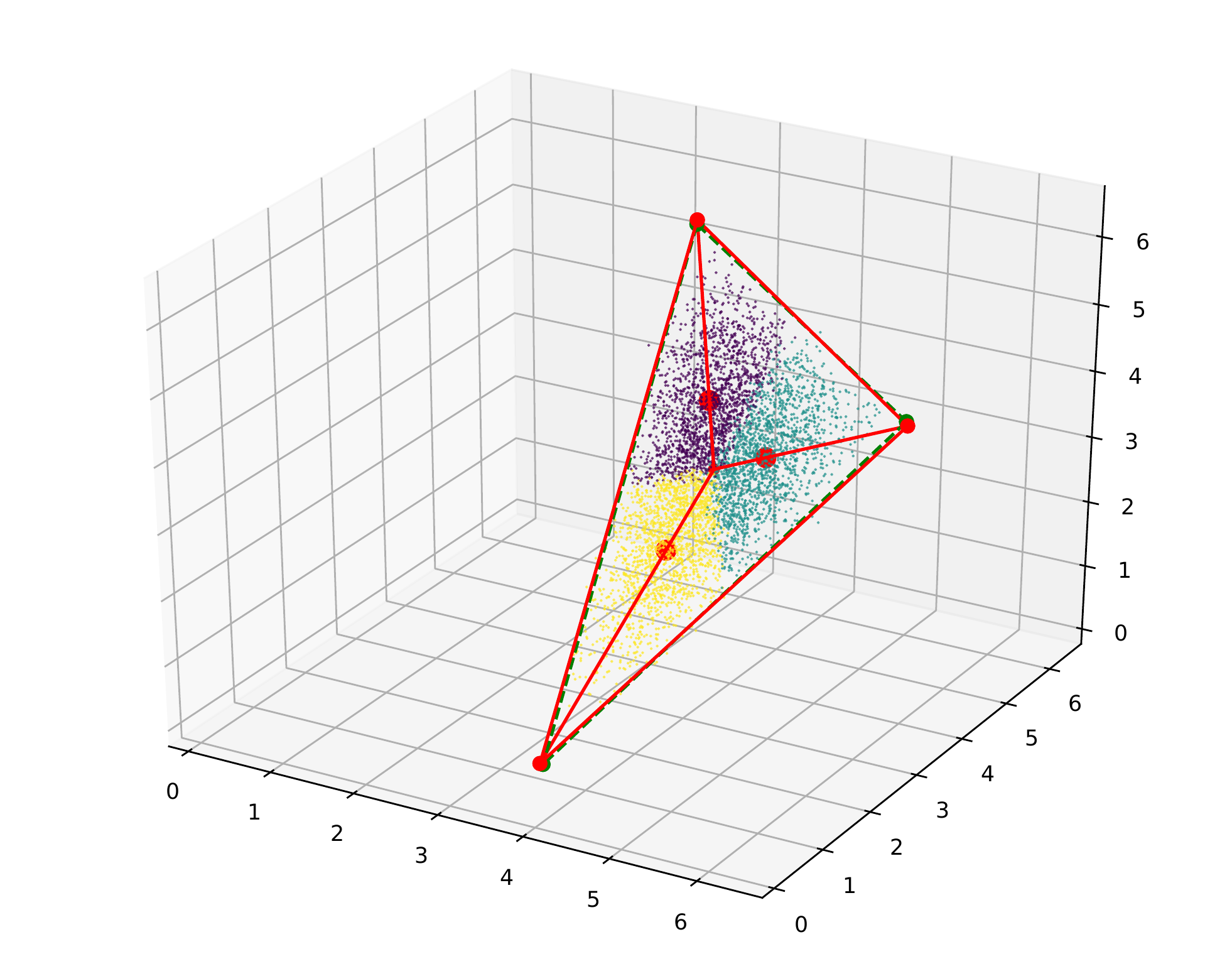}
  \caption{VLAD; time < 1s}
  \label{fig:gam}
\end{subfigure}
\caption{Toy simplex learning: $n=5000,D=3,K=3,\alpha=2.5,\sigma=0.1$.}
\label{fig:triangle}
\vskip -0.2in
\end{figure*}
We proceed to formally describe Dirichlet Simplex Nest as a generative model. Let $\beta_1,\ldots, \beta_K \in \mathcal{S}$ be $K$ elements in a $D$-dimensional vector space $\Scal$, and define $\Bscr = \textrm{Conv}(\beta_1,\ldots, \beta_K)$ as their convex hull. When $K \leq D+1$, $\Bscr$ is a simplex in general positions. Next, for each $i=1,\ldots, n$, generate a random vector $\mu_i \in \Bscr$ by taking $\mu_i := \sum_{k=1}^K \theta_{ik} \beta_k$, where the corresponding coefficient vector $\theta_i = (\theta_{i1},\ldots,\theta_{iK}) \in \Delta^{K-1}$
is generated by letting $\theta_i \thicksim \text{Dir}_K(\alpha)$ for some concentration parameter $\alpha \in \mathbb{R}_+^K$. 
Now, given $\mu_i$ the data point $x_i$ is generated by $x_i| \mu_i \thicksim F(\cdot\mid\mu_i)$, where $F$ is a given probability kernel such that $\Ex[x_i\mid\theta_i] = \mu_i$ for any $i=1,\ldots,n$.

\paragraph{Relation to existing models} The DSN encompasses several existing models in the literature. If we set $\Scal := \Delta^{D-1}$ and likelihood kernel $F(\cdot)$ to Multinomial, then we recover the LDA model \cite{blei2003latent}. Other specific instances include Gaussian-Exponential \cite{schmidt2009bayesian} and Poisson-Gamma models \cite{cemgil2009bayesian}. 

Estimating $\Bscr$ is a challenging task for the general Dirichlet Simplex Nest model. Taking the perspective of Bayesian inference, a standard MCMC implementation for the DSN is likely computationally inefficient. In the case of LDA, as noted in~\citet{yurochkin2016geometric}, the inefficiency of posterior inference can be traced to the need for approximating the posterior distributions of the large number of latent variables representing the topic labels. With the DSN model, we bypass the representation of such latent variables by integrating them out, but doing so at the cost of losing conjugacy. An alternative technique is variational inference (cf. \citet{Blei2017Variational,paisley2014bayesian}). While very fast, this powerful method may be inaccurate in practice and does not carry a strong theoretical guarantee. 

\paragraph{Relation to NMF and archetypal analysis} The DSN provides a probabilistic justification for these methods, which often impose an additional geometric condition on the model known as {\it separability} that identifies the model parameters in a way that permits efficient estimation \cite{Donoho2003When,Arora2012Computing,Gillis2012Fast}. Separability is somewhat related to a control on the Dirichlet's concentration parameter $\alpha$, by setting $\alpha$ be sufficiently small. The DSN allows for a probabilistic description of the nature of the separation.  Moreover, by addressing \emph{also} the case where $\alpha$ is large, the DSN modeling provides an arguably more effective approach to archetypal analysis and non-negative matrix factorization for \emph{non-separable} data. We remark that an approach proposed by \citet{huang2016AnchorFree} also permits a more general geometric identification condition called sufficiently scattered, but this generality comes at the expense of efficient estimation.

\paragraph{Geometric inference}
Geometric Dirichlet Means (GDM) algorithm of \citet{yurochkin2016geometric} is a geometric technique for estimating the (topic) simplex $\Bscr$ that arises in the Latent Dirichlet Allocation model. The basic idea of GDM is simple: performing the $K$-means clustering algorithm on the $n$
points $\mu_i$ (or their estimates) to obtain $K$ centroids. These centroids cannot be a good estimate for $\Bscr$'s vertices, but they provide reasonable directions toward the vertices. Starting from the simplex's estimated centroid, the GDM constructs $K$ line segments connecting to the $K$ centroids and suitably extends the rays to provide an estimate for the $K$ vertices. The GDM method is shown to be accurate when either $\Bscr$ is equilateral, or the Dirichlet concentration parameter $\alpha$ is very small, i.e., most of the points $\mu_i$s are concentrated near the vertices. The quality of the estimates deteriorates in the absence of such conditions.

The deficiency of the GDM algorithm can be attributed to several factors: first, for a general simplex, the $K$-means centroids and the simplex's vertices do not line up. Fortunately, we will see that they may be lined up in a straight line by a suitable affine transformation of the simplex structure. Second, the nature of the Dirichlet distribution on the simplex is not pro-actively exploited, including that of parameter $\alpha$. Third, typically $K\ll D$, the affine hull of $\Bscr$ is a very low-dimensional structure, a fact not utilized by the GDM algorithm. It turns out that these shortcomings may be overcome by a careful consideration of the geometric structure of the simplex and the Dirichlet distribution.

For illustrations, we consider a toy problem of learning extreme points of simplex $\Bscr$, given Gaussian data likelihood  $x_i|\mu_i \thicksim \mathcal{N}(\mu_i,\sigma^2I_D)$ and $D=K=3$. The triangle is chosen to be non-equilateral and Dirichlet concentration parameter is set to $\alpha=2.5$. Figure \ref{fig:gdm} illustrates the deteriorating performance of the GDM. In Figure \ref{fig:xray}, we also observe Xray \cite{kumar2013fast}, another recent NMF algorithm, failing to solve the problem, as the aforementioned separability assumption is violated for large $\alpha$. On the other hand, Figure \ref{fig:stan-hmc} demonstrates the high accuracy of the posterior mean obtained by Hamiltonian Monte Carlo (HMC) \cite{neal2011mcmc,hoffman2014no} implemented using Stan \cite{carpenter2017stan}, albeit at the cost of 10 minutes training time. Lastly our new algorithm (VLAD) in Fig. \ref{fig:gam}, exhibits an accuracy comparable to that of the HMC and the run-time of the GDM algorithm.

\section{Inference of the Dirichlet Simplex Nest}
\label{sec:sn_estimation}

\subsection{Simplicial Geometry}

In order to motivate our algorithm, we elucidate the geometry of the DSN through the concept of Centroidal Voronoi Tessellation (CVT) \cite{du1999centroidal} of a simplex $\mathscr{B}$, a convex subset of $D$-dimensional metric space $\Scal$.

\begin{definition}[Centroidal Voronoi Tessellation]
\label{def:cvt}
Let $\Omega\subset\Scal$ be an open set equipped with a distance function $d$ and a probability density $\rho$. 
For a set of $K$ points $c_1,\dots,c_K$, the Voronoi cell corresponding to $c_k$ is the set
\[
V_k = \{x\in\Omega: d(x,c_k) < d(x,c_l)\text{ for any }l\ne k\}.
\]
The collection of Voronoi cells $V_1,\dots,V_K$ is a tessellation of $\Omega$; i.e.\ the cells are disjoint and $\cup_k V_k = \Omega$. If $c_1,\dots,c_K$ are also the centroids of their respective Voronoi cells, i.e., 
\[
c_k = \frac{1}{\int_{V_k} \rho(x)\textrm{d}x}\int_{V_k}x\rho(x) \textrm{d}x
\]
the tessellation is a Centroidal Voronoi Tessellation.
\end{definition}

CVTs are special: any set of $k$ points induces a Voronoi tessellation, but these points are generally not the centroids of their associated cells. One can check that a CVT minimizes
\[
J(c_1,\dots,c_K) = \int_{V_k}d(x,c_k)^2\rho(x)
\textrm{d}x.
\]
It is a fact that $J$ has a unique global minimizer as long as $\rho$ vanishes on a set of measure zero, the Voronoi cells are convex, and the distance function is convex in each argument \cite{du1999centroidal}. 
Moreover, 
it can be seen that
the centroids of the CVT of an equilateral simplex equipped with the $\text{Dir}_K(\alpha)$ distribution fall on the line segments between the centroid of the simplex and the extreme points of the simplex, but this is not the
case when the simplex shape is non-equilateral (cf. Fig. \ref{fig:gdm}).

The following lemma formalizes the aforementioned insight to a simplex of arbitrary shape $\Bscr$ by considering a suitably modified distance function $d(\cdot,\cdot)$ of the CVT.  (In Fig. \ref{fig:gam}, the blue, purple and yellow dots are the sample versions of the Voronoi cells of the CVT under the new distance function and the corresponding centroids are in red.) 

\begin{lemma}
\label{CVT_general_simplex}
Let $B\in\reals^{D\times K}$ denote the matrix form of simplex $\mathscr{B}$. Suppose it has full (column) rank, equipped with distance function $\|\cdot\|_{(BB^T)^\dagger}$ and the probability distribution $\bbP_B$ defined as
\[
\bbP_B(S) = \textrm{Prob}(\{\theta\in\Delta^{K-1}:B\theta\in S\}),
\]
where $\theta$ is distributed by symmetric Dirichlet density $\rho_\alpha :=\text{Dir}_K(\alpha)$, 
for any $S\subset\intr(\mathscr{B})$, and $A^{\dagger}$ denotes a pseudo-inverse of $A$.
The centroids of its CVT fall on the line segments connecting the centroid of $\mathscr{B}$ to $\beta_1,\dots,\beta_K$.
\end{lemma}

\begin{proof}
Let $c_1,\dots,c_K$ and $V_1,\dots,V_K$ be the centroids and cells of the CVT of $\Delta^{K-1}$ equipped with Euclidean distance and $\text{Dir}_K(\alpha)$ density $\rho_\alpha$. It suffices to verify that $Bc_1,\dots,Bc_K$ and $BV_1,\dots,BV_K$ are the centroids and cells of the CVT of $\mathscr{B} = B\Delta^{K-1}$. By a change of variables formula,
\begingroup\makeatletter\def\f@size{9}\check@mathfonts
\[
\begin{aligned}
&\argmin\biggr \{\frac{\int_{BV_k}\|x - Bv\|_{(BB^T)^\dagger}^2\rho_\alpha(B^{\dagger}x)|\det(B^{\dagger})|\textrm{d}x}{\int_{V_k}\rho_\alpha(B^{\dagger}x)|\det(B^{\dagger})|\textrm{d}x}:v\in V_k\biggr\} \\
&= \argmin\biggr\{\frac{\int_{V_k}\|B\theta - Bv\|_{(BB^T)^\dagger}^2\rho_\alpha(\theta)\textrm{d}\theta}{\int_{V_k}\rho_\alpha(\theta)\textrm{d}\theta}:v\in V_k\biggr\} \\
&= \argmin\biggr\{\frac{\int_{V_k}\|\theta - v\|_2^2\rho_\alpha(\theta)\textrm{d}\theta}{\int_{V_k}\rho_\alpha(\theta)\textrm{d}\theta}:v\in V_k\biggr\},
\end{aligned}
\] 
\endgroup
which we recognize as the centroids of the CVT of $\Delta^{K-1}$ under $\ell_2$ metric. Since $\Delta^{K-1}$ is a standard simplex and therefore equilateral, the centroids of the CVT of equilateral simplex fall on the line segments connecting the centroid of the simplex to its extreme points.
\end{proof}
Lemma \ref{CVT_general_simplex} suggests an algorithm to estimate the extreme points of $\mathscr{B}$. First, estimate the centroids of the CVT of $\Bscr$ (equipped with scaled Euclidean norm $\|\cdot\|_{(BB^T)^\dagger}$) and search along the rays extending from the centroid of $\Bscr$ through the CVT centroids for the simplicial vertices.
\subsection{The Voronoi Latent Admixture (VLAD) Algorithm} \label{algorithm}
We  first consider the noiseless problem, $F(\cdot\mid\mu) = \delta_\mu$. That is, $x_i=\mu_i$s are observed. 
In this case, Lemma~\ref{CVT_general_simplex} suggests estimating the CVT centroids by scaled $K$-means optimization:
\begin{equation}
\argmin_{c_1,\ldots,c_K}\bigl\{{\textstyle\frac12\sum_{k=1}^K\sum_{x_i\in V_k}(x_i - c_k)^T(BB^T)^\dagger(x_i - c_k)}\bigr\},
\label{eq:clusteringUncenteredData}
\end{equation}
Unfortunately, the scaled Euclidean norm $\|\cdot\|_{(BB^T)^\dagger}$ is unknown. We propose an equivalent approach that does not depend on knowledge of $BB^T$.

In the noiseless case, observe that the population covariance matrix of the samples takes the form $\Sigma = BSB^T$, where $S$ is the covariance matrix of a $\text{Dir}(\alpha)$ random variable on $\Delta^{K-1}$. By the standard properties of the $\text{Dir}(\alpha)$ distribution, it can be seen that $S = \frac{1}{K(K\alpha + 1)}P$, where $P = I_K - \frac1K\ones_K\ones_K^T$ is the centering matrix. Hence, knowledge of $\Sigma$ will be sufficient because the centered data points $x$ fall in $\lspan(\Sigma) = \lspan(BPB^T)$: For each $(\theta,x)$ pair,
\begin{equation}
\bar{x} := \underbrace{B\theta}_{x} - \underbrace{\textstyle\frac1KB\ones}_{\Ex[x]} = B\theta - {\textstyle\frac1K}B\ones(\underbrace{\ones^T\theta}_{=1}) = BP\theta := B\bar{\theta}.
\label{eq:centeredDataPoints}
\end{equation}
This suggests that the centroids of the CVT may be recovered by clustering the centered data points in the $\|\cdot\|_{\Sigma^{\dagger}}$-norm. This insight is formalized by
%
\begin{lemma}
\label{lemma:simplex_transform}
The centroids of the CVT of simplex $\mathscr{B}$ under $\|\cdot\|_{(BB^T)^{\dagger}}$-norm are given by $\{c^*_k + c_0 | k=1,\ldots, K\}$, where $(c^*_1,\dots,c^*_K)$ solves the minimization
\begin{equation}
\min_{ \substack{c_1,\ldots,c_K \\ V_1,\ldots,V_K}} \frac12\sum_{k=1}^K \bigintssss_{x\in BV_k}(\bar{x} - c_k)^T\Sigma^\dagger(\bar{x} - c_k) \rho(x) \mathrm{d}x
\label{eq:clusteringCenteredData}
\end{equation}
\comment{
\begin{equation}
\begin{aligned}
& \bigl\{{\textstyle\frac12\sum_{k=1}^K\bigints_{x\in V_k}(\bar{x}_i - \bar{c}_k)^T\Sigma^\dagger(\bar{x}_i - \bar{c}_k)}\bigr\} \\
&=\arg\min\bigl\{{\textstyle\frac12\sum_{k=1}^K\sum_{x_i\in V_k}(\bar{x}_i - c_k)^T\Sigma^\dagger(\bar{x}_i - c_k)}\bigr\}.
\label{eq:clusteringCenteredData}
\end{aligned}
\end{equation}
}
and $c_0 = \int x \rho(x) \textrm{d}x$ is the centroid of simplex $\mathscr{B}$.
\end{lemma}

\begin{proof}
We first show that \eqref{eq:clusteringCenteredData} is equivalent to (unscaled) $K$-means clustering on $\Delta^{K-1}$.  
\comment{
We have
\begin{equation}
\begin{aligned}
&\frac12\sum_{k=1}^K\int_{x \in V_k}(\bar{x} - c_k)^T\Sigma^\dagger(\bar{x} - c_k)= \\
& \frac12\sum_{k=1}^K\int_{x \in V_k}(\bar{x} - c_k)^T(BPB^T)^\dagger(\bar{x} - c_k),
\end{aligned}
\end{equation}
}
Note that $\Sigma = \delta BPB^T$ for some $\delta > 0$.
Without loss of generality, we restrict to $c_k$'s in $\lspan\{BPB^T\}$. Write $c_k = BPv_k$ for $v_k \in \mathbb{R}^K$, for $k=1,\ldots, K$.
Recalling \eqref{eq:centeredDataPoints} and the fact $P$ is a projector, 
\begin{eqnarray}
& &(1/\delta)\textstyle\sum_{k=1}^K \int_{x\in BV_k}  (\bar{x} - c_k)^T\Sigma^\dagger(\bar{x} - c_k) \rho(x) \mathrm{d}x \nonumber\\
&=&{\textstyle\sum_{k=1}^K \int_{\theta\in V_k} }(\bar{\theta} - v_k)^TPB^T\Sigma^\dagger BP(\bar{\theta} - v_k) \rho_{\alpha}(\theta) \mathrm{d}\theta\nonumber\\
&=&{\textstyle\sum_{k=1}^K \int_{\theta\in V_k} }(\bar{\theta} - v_k)^T P(\bar{\theta} - v_k) \rho_{\alpha}(\theta) \mathrm{d}\theta\nonumber \\
&=&{\textstyle\sum_{k=1}^K \int_{\theta\in V_k} }\|\bar{\theta} - Pv_k\|_2^2  \rho_{\alpha}(\theta) \mathrm{d}\theta.
\end{eqnarray}
Since $\theta$ is distributed by the symmetric Dirichlet $\rho_\alpha = \text{Dir}(\alpha)$ on $\Delta^{K-1}$, the last equality entails that the optimal $v_k$'s are the points which represent the barycentric coordinate of the centroids of the CVT of $\Delta^{K-1}$. 
Thus, the optimal solution for $c_k = BPv_k$ represents the centroids of the CVT of simplex $\mathscr{B}$ under $\|\cdot\|_{(BB^T)^{\dagger}}$-norm (using the coordinating system that is centered at origin $c_0$). 
\end{proof}

We proceed to address the optimization \eqref{eq:clusteringCenteredData} applied to empirical data to arrive at Voronoi Latent Admixture (VLAD) algorithm in Algorithm \ref{algo:VLAD}. We   utilize the singular value decomposition (SVD) of the centered data points to simplify computation. Let $\bar{X}\in\reals^{n\times D}$ be the matrix whose rows are the centered data points and $\bar{X} = U\Lambda W ^T$ be its SVD. Each term in the objective of \eqref{eq:clusteringCenteredData} is equivalent to, with $\Sigma$ being replaced by its empirical version, $\Sigma_n=\frac{1}{n}W\Lambda^2 W^T$:
\begin{equation}
\begin{aligned}
&(\bar{x}_i - c_k)^T{\Sigma_n^\dagger}(\bar{x}_i - c_k)= \\ & n(u_i - \eta_k)^T\Lambda W^TW\Lambda^{-2}W^TW\Lambda(u_i - \eta_k) = n\|u_i - \eta_k\|_2^2,\nonumber
\end{aligned}
\end{equation}
where $\bar{x}_i = W\Lambda u_i$, and set $c_k=W\Lambda\eta_k $. Thus, instead of performing scaled $K$-means clustering in $\Scal$, it suffices to perform standard $K$-means in the low $(K-1)$ dimensional space. This yields a significant computational speed-up. After applying VLAD, the weights $\theta_i$'s can be obtained by projecting the data points onto $\mathscr{B}$ and compute the barycentric coordinates of the projected points. 
\begin{algorithm}[ht]
\caption{Voronoi Latent Admixture (VLAD)}
\label{algo:VLAD}
\begin{algorithmic}[1]
\REQUIRE data $x_1,\ldots,x_n$; $K$; extension parameter $\gamma$.
\ENSURE simplex vertices $\beta_1,\ldots,\beta_K$
\STATE $\widehat{c}_0\gets \frac{1}{n}\sum_i x_i $ \COMMENT{find data center}
\STATE $\bar{x}_i \gets x_i - \widehat{c}_0$, $i = 1,\dots,n$ \COMMENT{centering}
\STATE compute top $K-1$ singular factors of the centered data matrix $\bar{X}\in\reals^{n\times D}$: $\bar{X} = U\Lambda W^T$
\STATE $\eta_1,\ldots,\eta_K \gets \text{K-means}(u_1,\ldots,u_n$), where the $u_i$'s are the rows of $U\in\reals^{n\times (K-1)}$ 
\STATE $\widehat{c}_k \gets W\Lambda \eta_k + \widehat{c}_0$
\STATE $\widehat{\beta}_k \gets \widehat{c}_0 + \gamma(\widehat{c}_k - \widehat{c}_0)$
\end{algorithmic}
\end{algorithm}

It remains to estimate the extreme points $\beta_k$s given the CVT centroids $c_k$s. This task is simplified by two observations:
First, the CVT centroids reside on the line segment between the centroid of simplex $\mathscr{B}$ and its extreme points, per Lemma~\ref{CVT_general_simplex}. Thus we merely need to estimate the ratio of the distance from the extreme point to the centroids of $\Bscr$ and the distance from the CVT centroids to the centroid of $\Bscr$. Due to the symmetry of $\text{Dir}_K(\alpha)$ distribution on $\Delta^{K-1}$, this ratio is identical for all extreme points -- we refer to this ratio as the extension parameter $\gamma$.
Secondly, $\gamma$ does not depend on the geometry of $\mathscr{B}$, only that of the Dirichlet distribution. Thus, $\gamma$ can be easily estimated by appealing to a Monte Carlo technique on $\text{Dir}_K$. This subroutine is summarized in Algorithm \ref{algo:extensionParams}, provided that $\alpha$ is given.

\begin{algorithm}[ht]
\caption{Evaluating extension parameters}
\label{algo:extensionParams}
\begin{algorithmic}[1]
\STATE generate $\theta_1,\dots,\theta_m\sim\Dir_K(\alpha)$, where $m$ is the number of Monte Carlo samples
\STATE $v_1,\dots,v_K \gets \text{K-means}(\theta_1,\dots,\theta_m)$
\STATE $\gamma\gets \sqrt{K^2 - K}\left(\sum_{l=1}^K\|v_l - \frac1K\ones_K\|_2\right)^{-1}$
\end{algorithmic}
\end{algorithm}
%
%
%

\subsection{Estimating the Dirichlet Concentration Parameter}
\label{sec:alpha}
Next, we describe how to estimate concentration parameter $\alpha$ from the data, by employing a moment-based approach. Recall from the previous section that there is an one-to-one mapping between $\alpha$ and the extension parameter $\gamma$. For each $\alpha > 0$, let $\gamma(\alpha) > 0$ denote the corresponding extension parameter and $B(\gamma)\in\reals^{D\times K}$ the estimator of $B$ output by VLAD with extension parameter $\gamma$. In the absence of noise, the covariance matrix of the DSN model has the form $BS(\alpha)B^T$, where $S(\alpha)\in\reals^{K\times K}$ is the covariance matrix of a  $\text{Dir}(\alpha)$ random variable on $\Delta^{K-1}$. This suggests we estimate $\alpha$ by a generalized method of moments approach:
\begin{equation}
\hat{\alpha} = \argmin_{\alpha > 0}\|\hat{B}(\gamma(\alpha))S(\alpha)\hat{B}(\gamma(\alpha))^T - \hat{\Sigma}\|,
\label{eq:estimateAlpha}
\end{equation}
where $\hat{\Sigma}$ is the sample covariance matrix $\hat{\Sigma} = \frac1n \bar{X}^T\bar{X}$. We remark that there is no need to run VLAD multiple times to evaluate the objective in \eqref{eq:estimateAlpha} at multiple $\alpha$-values. After VLAD is run once, we may evaluate $\gamma(\alpha)$ for any value of $\gamma$ by affinely transforming the output of VLAD. Further, \eqref{eq:estimateAlpha} is a scalar optimization problem, so the computational cost of solving \eqref{eq:estimateAlpha} is negligible. 

In the presence of noise, the covariance matrix of the DSN model no longer has the form $BS(\alpha)B^T$. We need to add a correction term  to ensure a consistent estimator of $BS(\alpha)B^T$. For example, if the noise is Gaussian, a consistent estimator of $BS(\alpha)B^T$ is 
\[
\tilde{\Sigma} = \hat{\Sigma} - \hat{\sigma}^2I_D,
\]
where $\hat{\sigma}^2$ is an estimate of the noise variance. In Supplement~\ref{consistent estimator}, we give consistent estimators of $BS(\alpha)B^T$ for multinomial and Poisson noise. With a good estimator $\tilde{\Sigma}$ of $BS(\alpha)B^T$ in place, we replace $\hat{\Sigma}$ in \eqref{eq:estimateAlpha} by $\tilde{\Sigma}$ and then solve \eqref{eq:estimateAlpha} to obtain an estimate of $\alpha$. 

\section{Consistency and Estimation Error Bounds}
\label{sec:theory}
In this section we   establish consistency properties and error bound guarantees of the VLAD procedure.

For $c=(c_1,\dots,c_K)\in\reals^{K\times D}$, define $\phi_{A} : \mathbb{R}^D \times \mathbb{R}^{K \times D} \rightarrow \mathbb{R}$ as 
\[\textstyle
\phi_{A}(x,c)=\min_{k \in \{1,\dots,K\}} \|x-c_k\|^2_{A^{\dagger}}
\]
where $A$ is a positive semidefinite matrix. Recall $\Sigma$ as the covariance matrix of the data generating distribution 
and $\Sigma_n$ its empirical counterpart. In the algorithm, we work with the best rank $K-1$ approximation of $\Sigma_n$, which we denote by $(\Sigma_n)^K$. Let $\Qdist$ denote the distribution for $\mu_i$s. Recall that $X_i | \mu_i \sim F(\cdot|\mu_i)$. Let $\mathbb{P}$ be the induced distribution corresponding to $\tilde{X}_i$, which is the projection of $X_i$ on the affine space of dimension $K-1$ spanned by the top $K-1$ eigenvectors of $\Sigma$. We also use $\mathbb{P}_n$ to denote the empirical distribution of the data represented by random variables $X_i$.

Since $K$-means clustering is a subroutine of our algorithm, we expect at least some sort of condition requiring that the $K$-means clustering routine be well-behaved in some sense. To that end we need the following standard condition on the population $K$-means objective (cf. \citet{pollard1981strong}).
\begin{enumerate}
\item[(a.1)] Pollard's regularity criterion (PRC): The Hessian matrix of the function 
$c \mapsto  \mathbb{Q} \phi_{BSB^T}(\cdot,c)\text{ evaluated at }c^*$
for all optimizers $c^*$ of $\mathbb{Q} \phi_{BSB^T}(\cdot,c)$ is positive definite, with minimum eigenvalue $\lambda_0>0$.
\end{enumerate}
It turns out that this will be all we need for the following theorem in the noiseless setting, where we have $\Sigma=BSB^T= (\Sigma)^K$ has rank $K-1$ and so, $\mathbb{P} = \mathbb{Q}$ and $\tilde{X_i} \overset{\mathscr{L}}{=} X_i$. 
\begin{theorem}
\label{thm:noiselessConsistency} Consider the noiseless setting, i.e., $F(\cdot\mid\mu) = \delta_\mu$.
Suppose that $\Bscr= \conv(\beta_1,\ldots,\beta_K)$ is the true topic simplex, while
$(\beta_{1n},\dots,\beta_{Kn})$ are the vertex estimates obtained by VLAD algorithm. 
Moreover, assume the error due to Monte Carlo estimates of the extension parameter is negligible.
Provided that condition (a.1) holds, 
\begin{eqnarray}
\min_{\pi}\|(\beta_{\pi_{(1)}n},\dots,\beta_{\pi_{(K)}n}) - (\beta_1,\dots,\beta_K)\| = O_{\mathbb{P}}(n^{-1/2}) \nonumber,
\end{eqnarray}
where the minimization is taken over all permutations $\pi$ of $\{1,\ldots,K\}$.
\end{theorem}
Note that the constant corresponding to the rate $O_{\mathbb{P}}(n^{-1/2})$ is dependent on the  Hessian matrix of the function $c \mapsto  \mathbb{P} \phi_{\Sigma}(\cdot,c)$. The proof for Theorem~\ref{thm:noiselessConsistency} is in Supplement~\ref{Proof of Theorem 1}.

In general, $F(\cdot\mid\mu)$ is not degenerate. Due to the presence of "noise" in the $K-1$ SVD subspace, the estimates of the CVT centroids may be inconsistent, which entails inconsistency of the VLAD's estimate for $\Bscr$. The following theorem provides an error bound in the general setting.
We need a strengthening of Pollard's Regularity Criterion. Let $(\Sigma)^K$ denote the best $K-1$ rank approximation of $\Sigma$ with respect to the Frobenius norm. Assume: 
\begin{enumerate}
\item[(a.2)]  The Hessian matrix of the function 
$c \mapsto  \mathbb{P} \phi_{(\Sigma)^K}(\cdot,c)$ evaluated at $c^*$
for all optimizers $c^*$ of $\mathbb{P} \phi_{(\Sigma)^K}(\cdot,c)$ is uniformly positive definite with minimum  eigenvalue  $ \lambda_0>0$, for all $(\Sigma)^K$ such that $(\Sigma-BSB^T) \leq \tilde{\epsilon} I_D $, for some $\tilde{\epsilon}>0$.
\end{enumerate}
The noise level is formalized by the following conditions:
\begin{enumerate}
\item [(b)] There is $\epsilon_0>0$ such that $ \epsilon_0 I_D- Cov(X|\theta)$ is positive semi-definite uniformly over $\theta \in \Delta^{K-1}$.
\item [(c)]There exists $M_0$ such that for all $M> M_0$, $\int_{\mathcal{B}(\sqrt{M},c_0)^c}\|x-c_0\|^2_2 g(x) \mathrm{d}x \leq \frac{k_1}{M}$,
for some universal constant $k_1$, where $\mathcal{B}(\sqrt{M},c_0)$ is a ball of radius $\sqrt{M}$ around population centroid $c_0$ and $g(\cdot)$ is the density of $\mathbb{P}$ with respect to the Lebesgue measure on the $K-1$ dimensional space which contains the top $K-1$ eigenvectors of $BSB^T + \epsilon_0 I_D$.
\end{enumerate}
\begin{theorem}
\label{thm:consistency}
Suppose that $\Bscr= \conv(\beta_1,\ldots,\beta_K)$ is simplex corresponding to extreme points of the DSN. Let $(\beta_{1n},\dots,\beta_{Kn})$ be the corresponding extreme point estimates obtained by the VLAD algorithm. 
Assume the error in the Monte Carlo estimates of the extension parameter is negligible.
Provided that (a.2), (b) and (c) hold, then
\begin{eqnarray}
\label{eqn:thm consistency}
& \min_{\pi} \|(\beta_{\pi_{(1)}n},\dots,\beta_{\pi_{(K)}n}) - (\beta_1,\dots,\beta_K)\|_2 = \nonumber\\ & O\left(\sqrt{\epsilon_0^{1/3}/\lambda_0}\right) + O_{\mathbb{P}}(n^{-1/2}),
\end{eqnarray}
where $\pi$ ranges over permutations of $\{1,\dots,K\}$.
\end{theorem}
The constant corresponding to the rate $ O_{\mathbb{P}}(n^{-1/2})$ in the theorem depends on the Hessian matrix of the function $c \mapsto \mathbb{P} \phi_{\Sigma}(\cdot,c)$; constant corresponding to the $O\left(\sqrt{ \epsilon_0^{1/3}/\lambda_0}\right)$ depends on the minimum and maximum eigenvalues of the matrix $BSB^T$. Proof is in Supplement \ref{sec:supp:theorem2_proof}.

The preceding results control the error incurred by the VLAD algorithm when the concentration parameter $\alpha$ is known. When $\alpha$ is unknown, our proposed solution in Section \ref{sec:alpha} performs well in both simulated and real-data experiments.
We do not know in theory whether the concentration parameter $\alpha$ is identifiable, we shall present empirical results in Supplement \ref{sec:supp:identif} which suggest identifiability. Assuming a condition which guarantees model identifiability, we can establish that the estimate obtained by the VLAD algorithm via \eqref{eq:estimateAlpha} is in fact consistent.
\begin{theorem}
\label{thm:alpha_consistency}
Assume that function $\varphi(\tilde{\alpha})=\frac{\gamma(\tilde{\alpha})^2}{K(K\tilde{\alpha} + 1)}$ is monotonically increasing in $\tilde{\alpha}$, where $\gamma(\tilde{\alpha})$ is the extension parameter corresponding to $\tilde{\alpha}$. Let $\alpha_0 \in \mathscr{C}$ be the true concentration parameter for some compact set $\mathscr{C}$. Let $\hat{\alpha}_n = \argmin_{\alpha \in \mathscr{C}}\| \hat{B}(\gamma(\alpha))S(\alpha)\hat{B}(\gamma(\alpha))^T - \tilde{\Sigma}_n\|$, where $\tilde{\Sigma}_n$ is a consistent estimator of $BS(\alpha)B^T$. Then,
\begin{eqnarray}
\|\hat{\alpha}_n -\alpha_0\| \overset{\mathbb{P}}{\longrightarrow} 0.
\end{eqnarray}
\end{theorem}
See Supplement \ref{sec:supp:theorem3_proof} for the proof.
\section{Experiments}
\label{sec:experiments}
The goal of our experimental studies is to demonstrate the applicability and efficiency of our algorithm for a number of choices of the DSN probability kernel: Gaussian, Poisson and Multinomial (i.e. LDA). 
We summarize all competing estimation procedures in our comparative study and their corresponding underlying assumptions in Table \ref{table:baselines}. 

We remark that Gibbs sampler \cite{griffiths2004finding}, Stan implementation of No U-Turn HMC \cite{hoffman2014no, carpenter2017stan} and Stochastic Variational Inference (SVI) \cite{hoffman2013stochastic} may be augmented with techniques such as empirical Bayes to estimate hyperparameter $\alpha$, although it may slow down convergence. We instead allow these baselines to use true values of $\alpha$ in all simulated experiments to their advantage; when latent simplex is of general geometry (i.e. non-equilateral), GDM \cite{yurochkin2016geometric} requires $\alpha \rightarrow 0$ to perform well, which is alike separability. Not all baselines are suitable for all three probability kernels, i.e. Gibbs sampler and SVI rely on (local) conjugacy and are only applicable in the LDA scenario; RecoverKL \cite{arora2013practical} is an algorithm that relies on a separability condition (i.e. anchor words) designed for topic models.

In simulated experiments we will consider both VLAD with estimated concentration parameter $\alpha$ following our results in Section \ref{sec:alpha} and VLAD trained with the knowledge of true data generating $\alpha$ (VLAD-$\alpha$). For real data analysis, we estimate the concentration parameter by \eqref{eq:estimateAlpha} and apply VLAD to a text corpus and 
stock market data set.

\begin{table}[t]
\caption{Baselines and required conditions}
\label{table:baselines}
\vskip -0.1in
\begin{center}
\begin{small}
\begin{tabular}{lccc}
\toprule
Method                                      & Conjugacy & True $\alpha$ & Separability \\
\midrule
VLAD (this work)                            & $\times$  & $\times$      & $\times$ \\
VLAD-$\alpha$ (this work)                   & $\times$  & $\surd$       & $\times$ \\
Gibbs \citeyearpar{griffiths2004finding}   & $\surd$   & $\surd^\star$ & $\times$ \\
Stan-HMC \citeyearpar{carpenter2017stan}   & $\times$  & $\surd^\star$ & $\times$ \\
SVI \citeyearpar{hoffman2013stochastic}    & $\surd$   & $\surd^\star$ & $\times$ \\
GDM \citeyearpar{yurochkin2016geometric}   & $\times$  & $\times$      & $\surd^\star$ \\
RecoverKL \citeyearpar{arora2013practical} & $\times$  & $\times$      & $\surd$ \\
SPA \citeyearpar{Gillis2012Fast}           & $\times$  & $\times$      & $\surd$ \\
MVES \citeyearpar{chan2009convex}          & $\times$  & $\times$      & $\surd$ \\
Xray \citeyearpar{kumar2013fast}           & $\times$  & $\times$      & $\surd$ \\
\bottomrule
\end{tabular}
\end{small}
\end{center}
\vskip -0.1in
\end{table}

\subsection{Comparative Simulation Studies}

\begin{figure*}[t]
\vskip -0.1in
\begin{subfigure}{.33\textwidth}
  \centering
  \includegraphics[width=\linewidth]{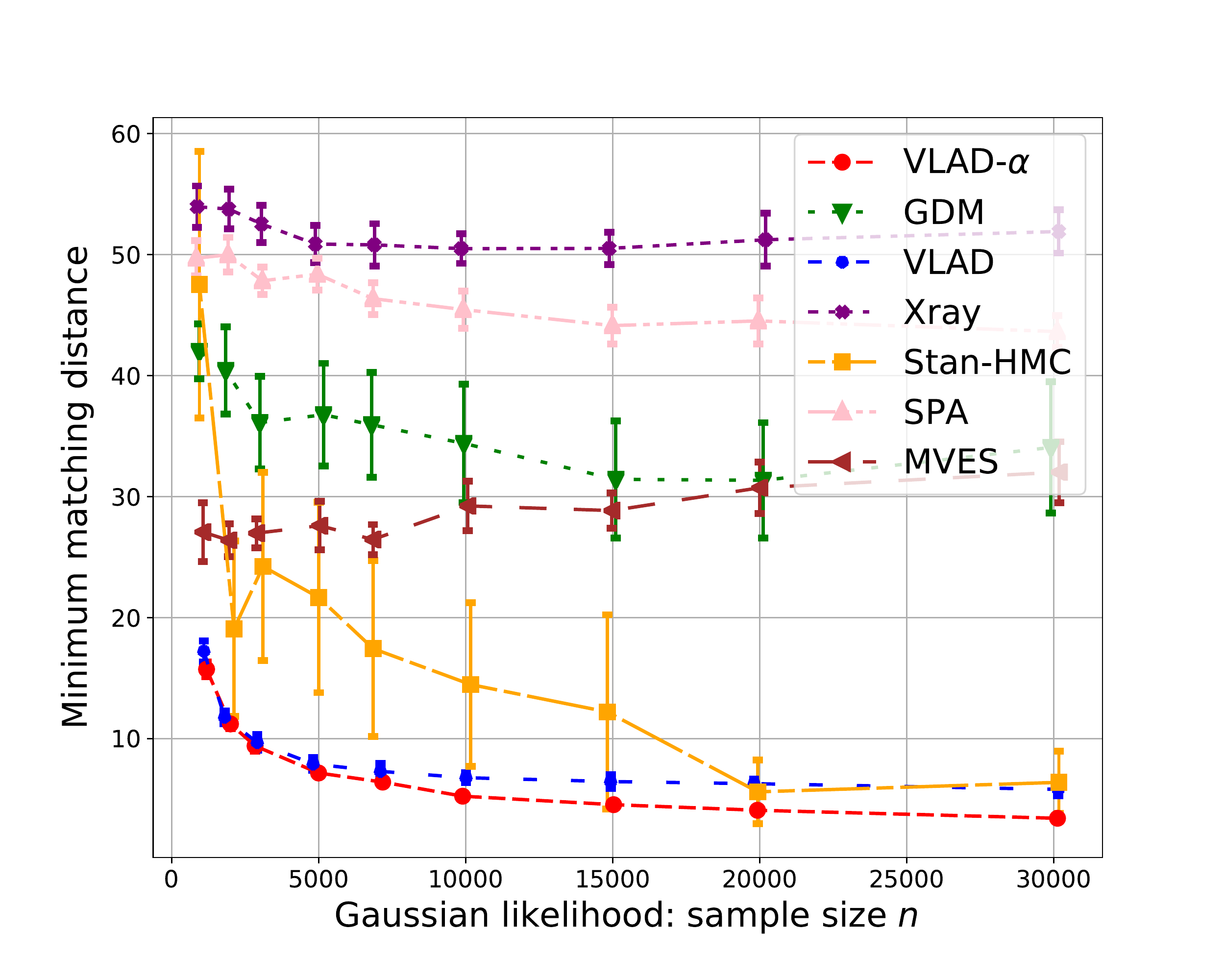}
  \vskip -0.1in
  \caption{Gaussian data}
  \label{fig:sample_size_gaus}
\end{subfigure}
\begin{subfigure}{.33\textwidth}
  \centering
  \includegraphics[width=\linewidth]{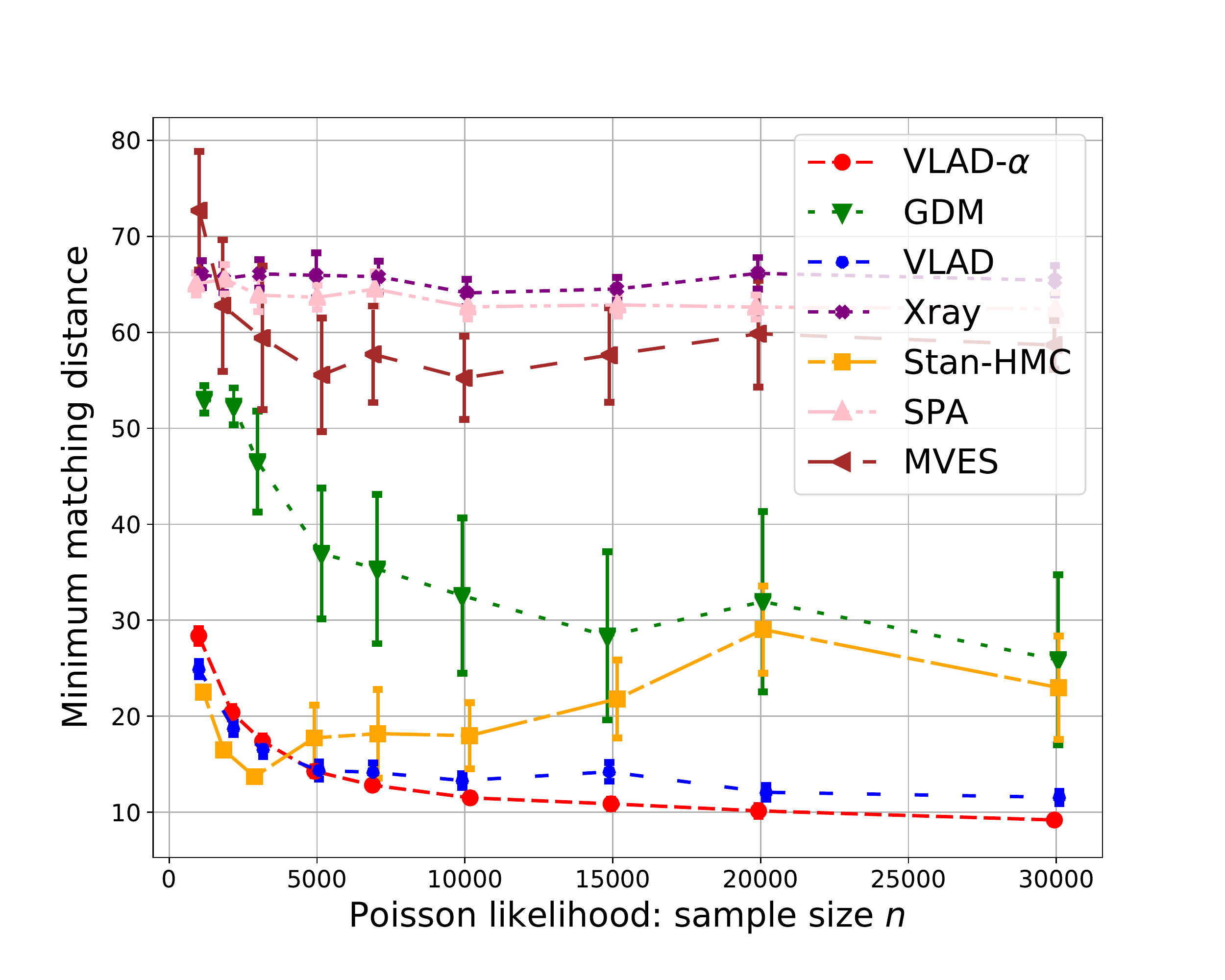}
  \vskip -0.1in
  \caption{Poisson data}
  \label{fig:sample_size_pois}
\end{subfigure}
\begin{subfigure}{.33\textwidth}
  \centering
  \includegraphics[width=\linewidth]{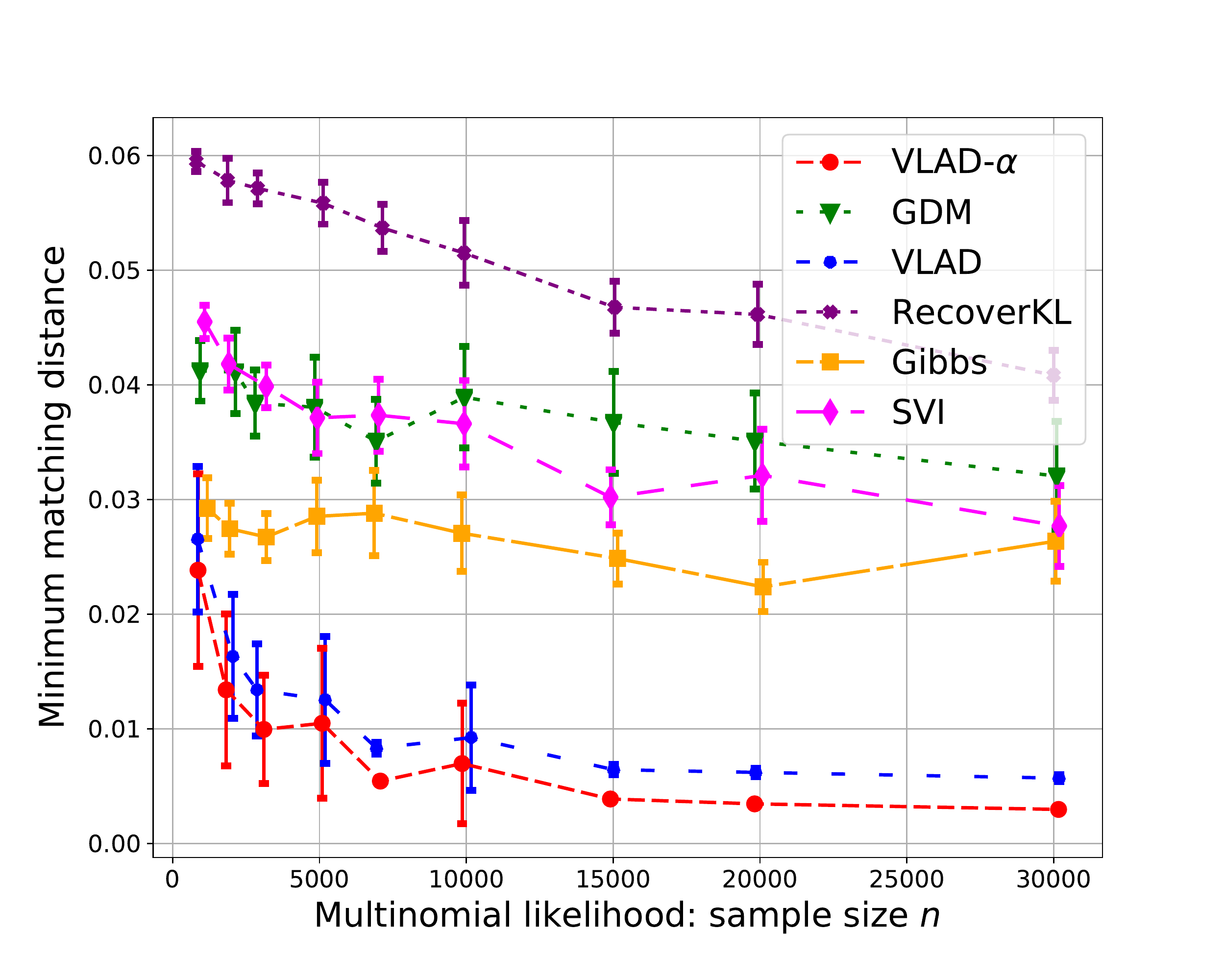}
  \vskip -0.1in
  \caption{Categorical data}
  \label{fig:sample_size_lda}
\end{subfigure}
\vskip -0.1in
\caption{Minimum matching distance for increasing $n$}
\label{fig:sample_size}
\end{figure*}

\begin{figure*}[t]
\vskip -0.1in
\begin{subfigure}{.33\textwidth}
  \centering
  \includegraphics[width=\linewidth]{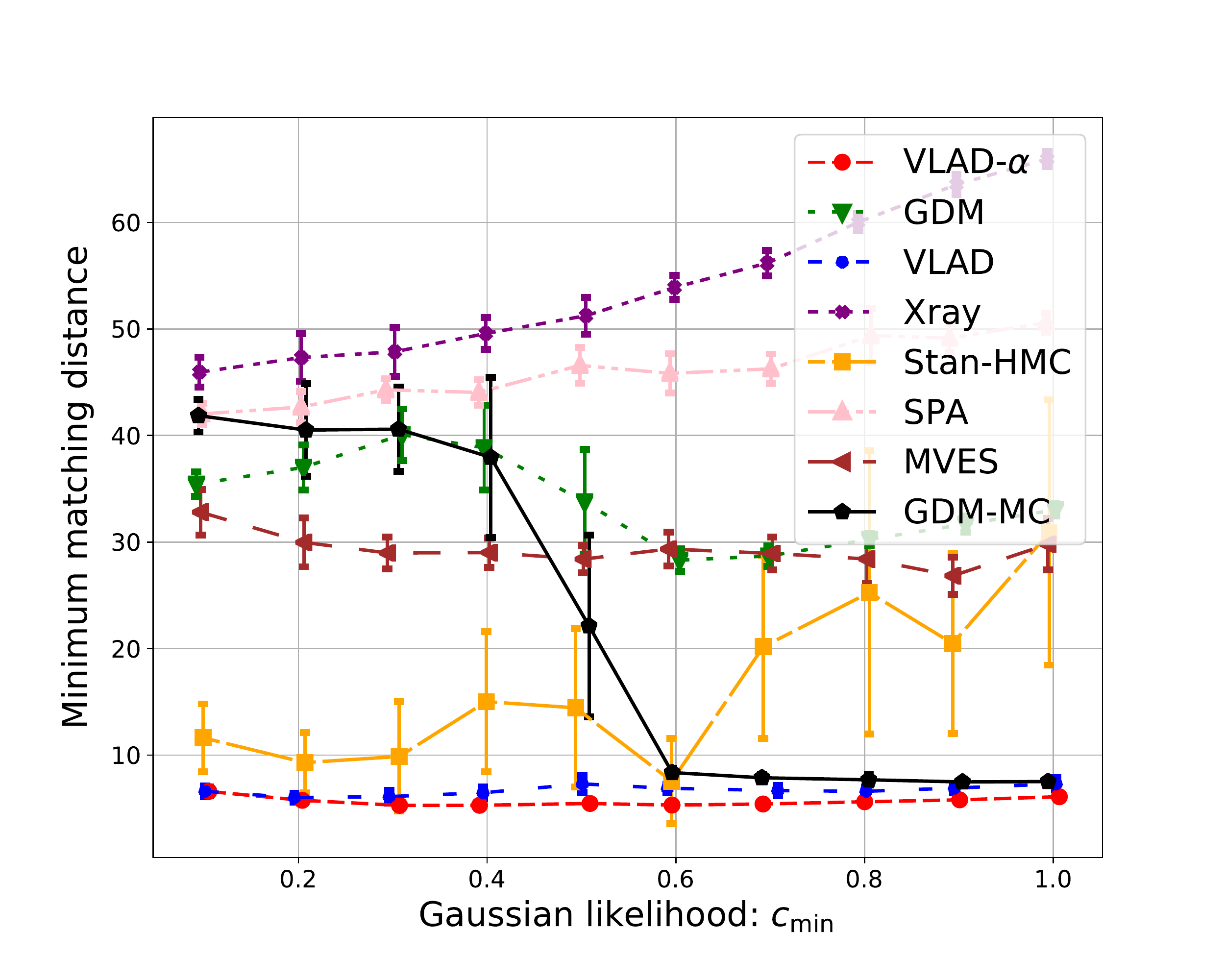}
  \vskip -0.1in
  \caption{Gaussian data}
  \label{fig:scale_gaus}
\end{subfigure}
\begin{subfigure}{.33\textwidth}
  \centering
  \includegraphics[width=\linewidth]{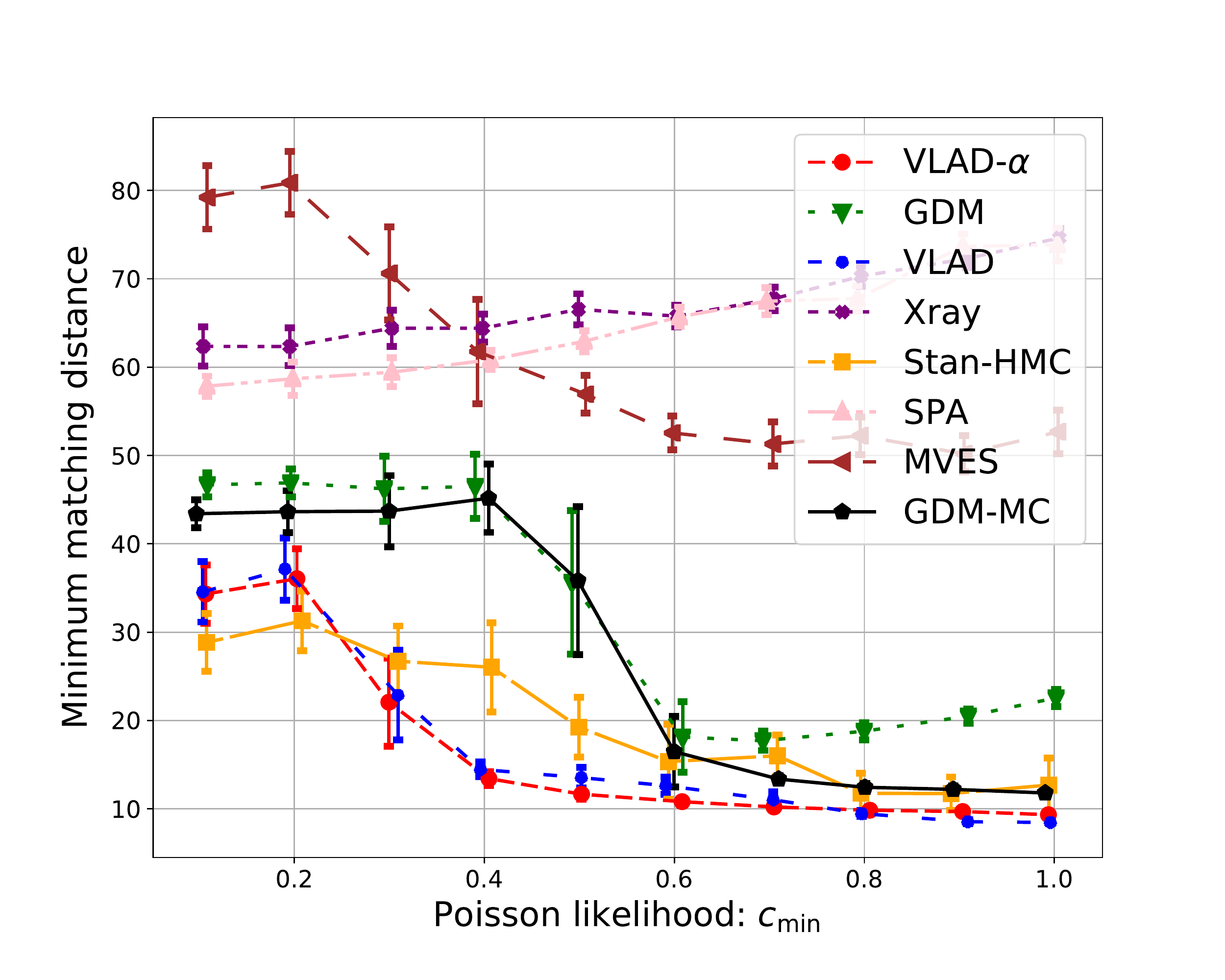}
  \vskip -0.1in
  \caption{Poisson data}
  \label{fig:scale_pois}
\end{subfigure}
\begin{subfigure}{.33\textwidth}
  \centering
  \includegraphics[width=\linewidth]{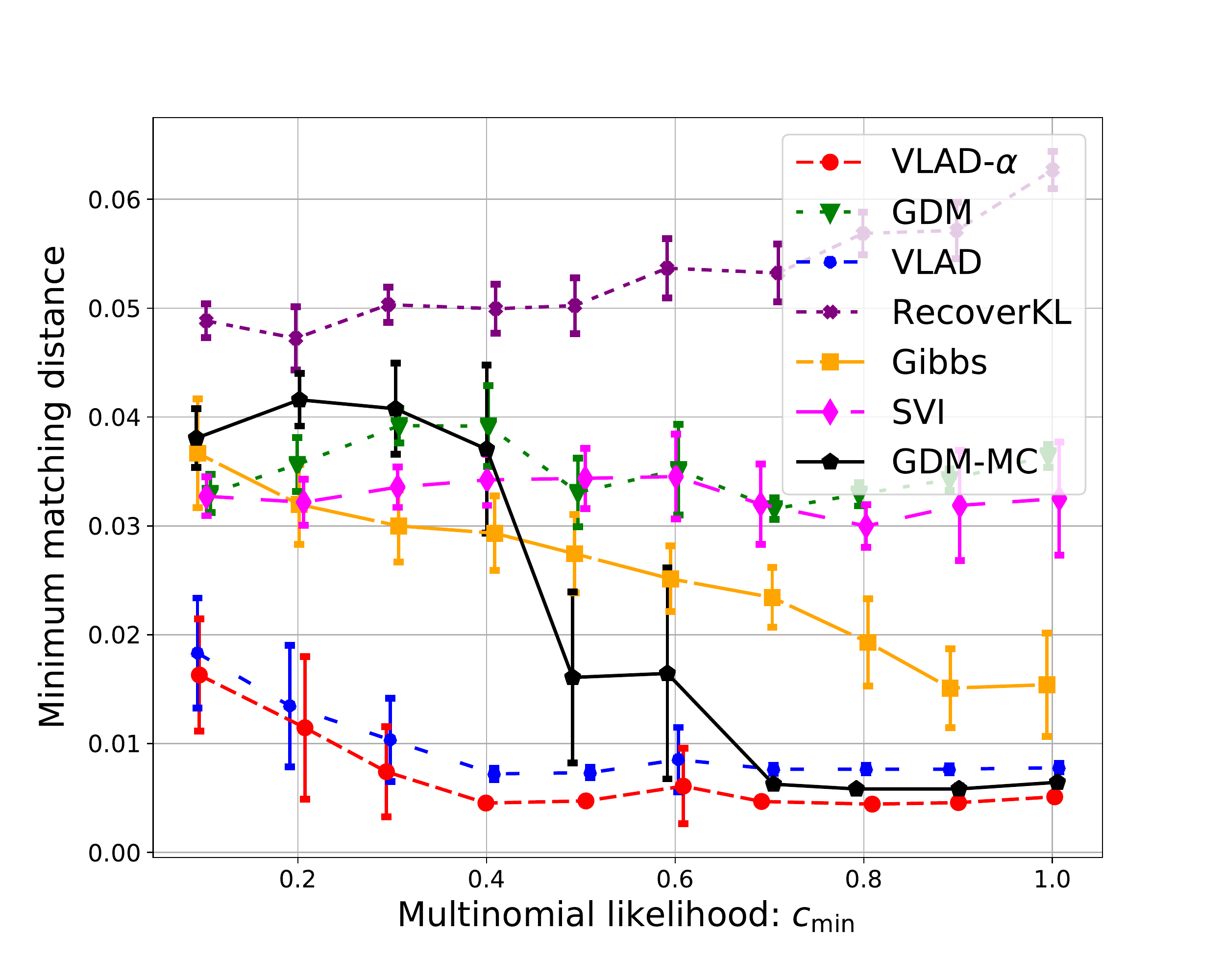}
  \vskip -0.1in
  \caption{Categorical data}
  \label{fig:scale_lda}
\end{subfigure}
\vskip -0.1in
\caption{Minimum matching distance for varying DSN geometry.}
\label{fig:scale}
\end{figure*}

\begin{figure*}[t]
\vskip -0.1in
\begin{subfigure}{.33\textwidth}
  \centering
  \includegraphics[width=\linewidth]{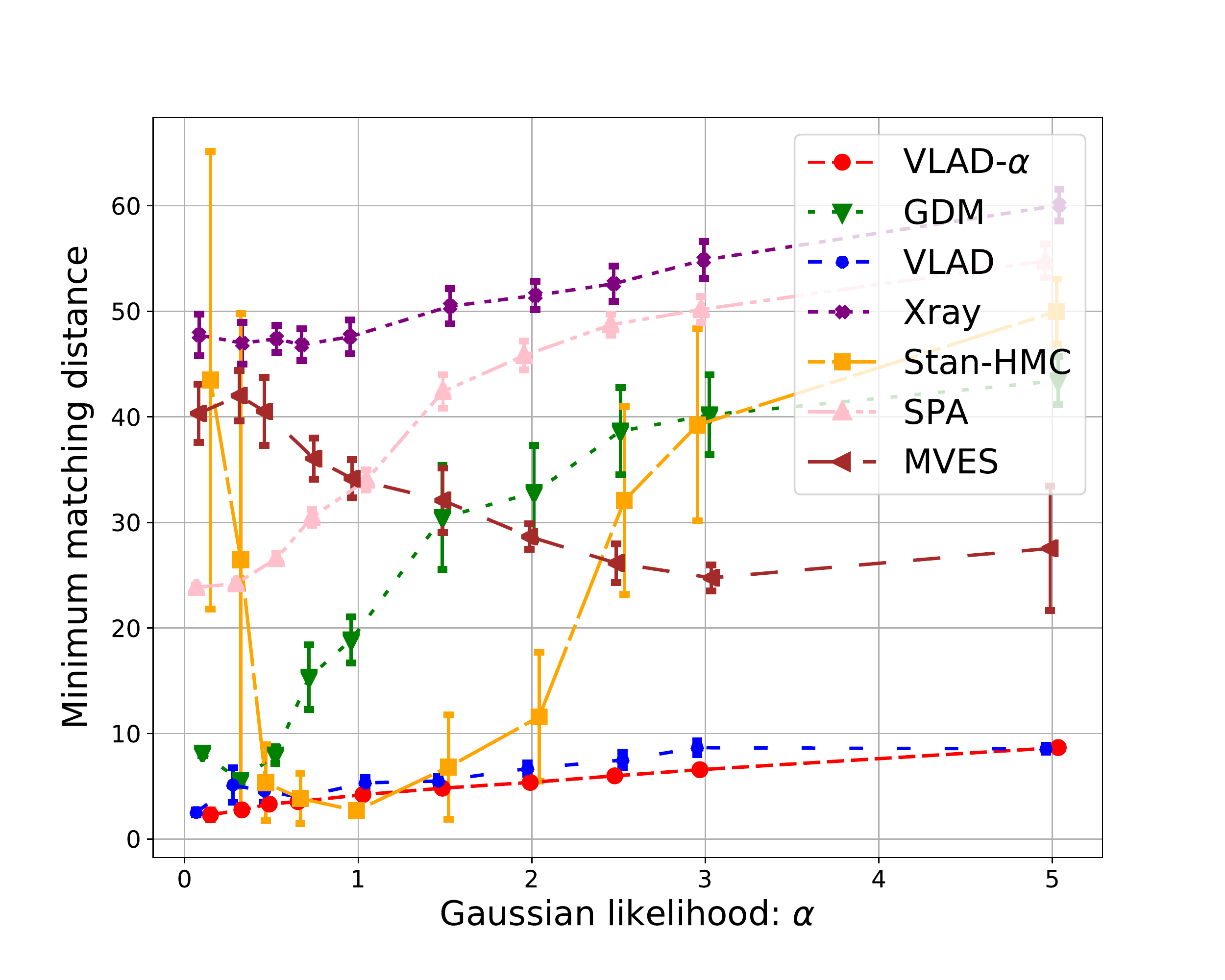}
  \vskip -0.1in
  \caption{Gaussian data}
  \label{fig:alpha_gaus}
\end{subfigure}
\begin{subfigure}{.33\textwidth}
  \centering
  \includegraphics[width=\linewidth]{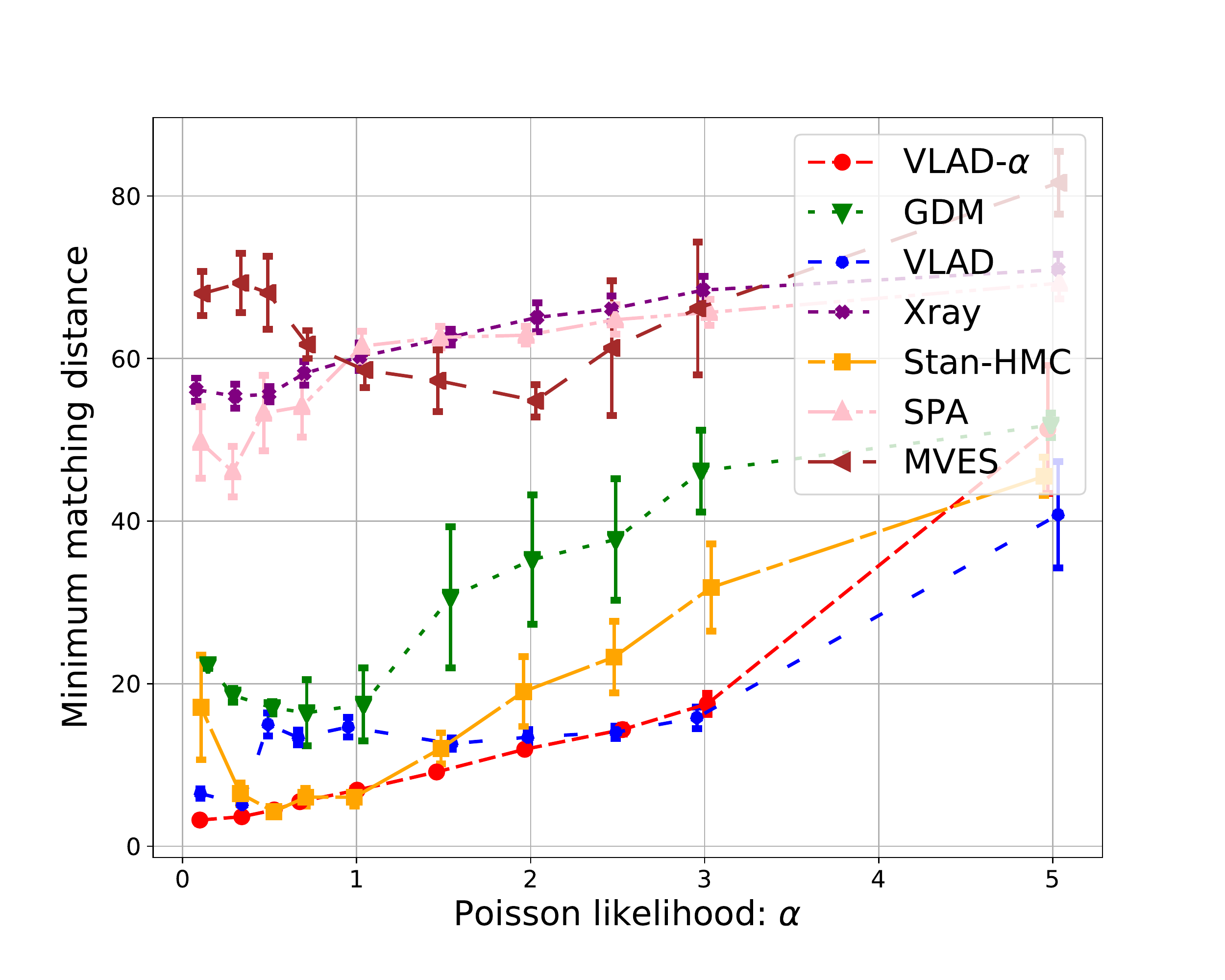}
  \vskip -0.1in
  \caption{Poisson data}
  \label{fig:alpha_pois}
\end{subfigure}
\begin{subfigure}{.33\textwidth}
  \centering
  \includegraphics[width=\linewidth]{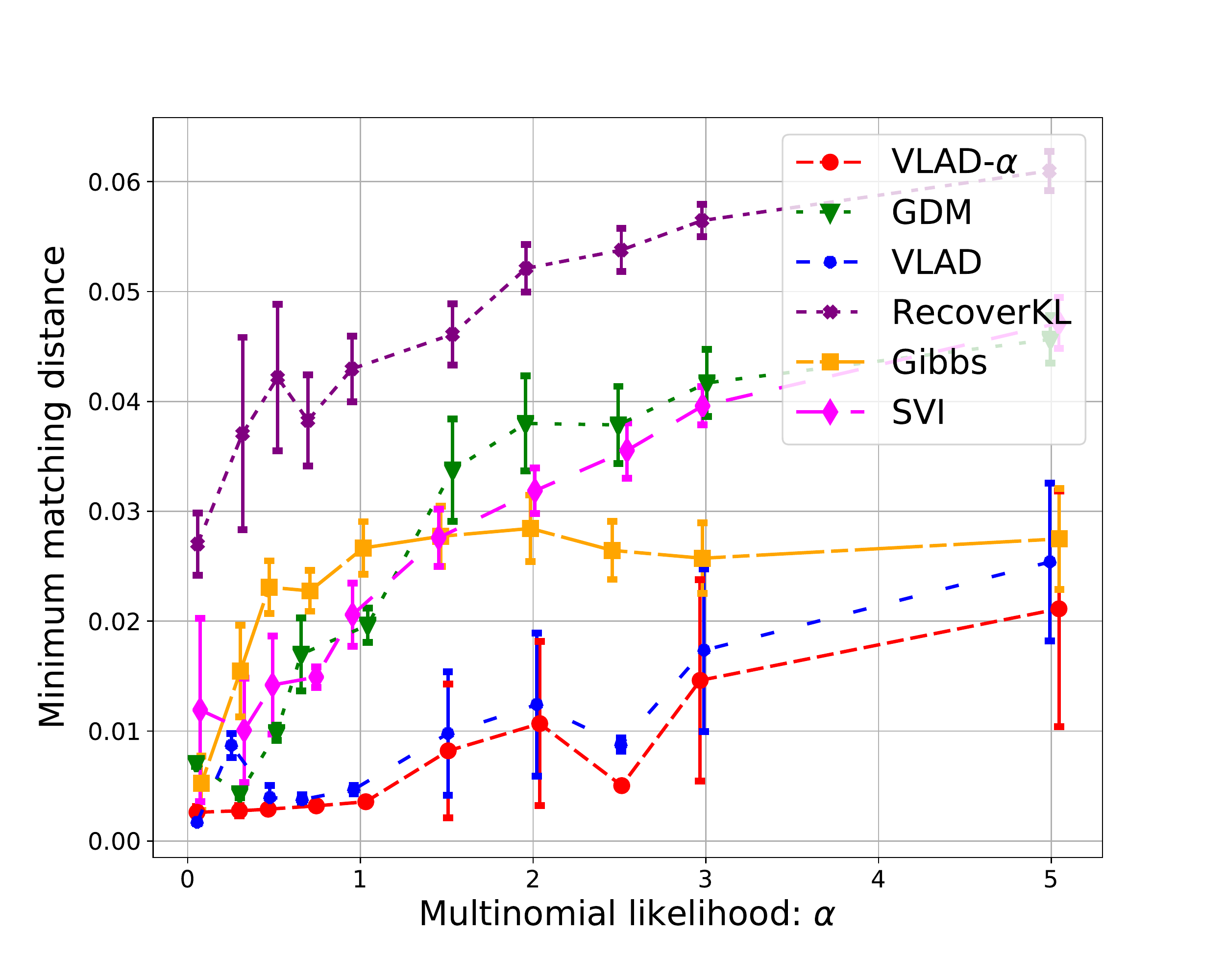}
  \vskip -0.1in
  \caption{Categorical data}
  \label{fig:alpha_lda}
\end{subfigure}
\vskip -0.1in
\caption{Minimum matching distance for increasing $\alpha$.}
\label{fig:alpha}
\end{figure*}

\begin{table*}[ht]
\caption{\hspace{1cm} NYT topic modeling (categorical data) \hspace{1.cm} || Stock data factor analysis (continuous data)}
\vskip -0.1in
\centering
\begin{tabular}{lrrr || crr}
\toprule
{} &  Perplexity & Coherence & Time & Frobenius norm & Volume & Time \\
\midrule
VLAD  &  1767 &  0.86 &      \textbf{6min} & 0.300 & \textbf{0.14} & \textbf{1s}   \\
GDM 		& 1777 &  0.88 &    30min     & 0.294 & 1499 & \textbf{1s} \\
Gibbs || HMC & \textbf{1520} & 0.80 & 5.3hours & 0.299 & 1.95 & 10min \\
RecoverKL || MVES & 2365 &  0.70 &  17min &  0.287 & $5.39 \times 10^9$ & 3min \\
SVI || SPA & 1669 & 0.81 &  40min & 0.392 & $3.31 \times 10^7$ & \textbf{1s} \\
\bottomrule
\end{tabular}
\label{table:data}
\end{table*}

\paragraph{Convergence behavior}
We investigate the convergence of the estimates of the DSN extreme points for the three likelihood kernels under the increasing sample size. The hyperparameter settings are $D=500,K=10,\alpha=2$ (for LDA vocabulary size $D=2000$). To ensure non-trivial geometry of the DSN we rescale extreme points towards their mean by uniform random factors between 0.5 and 1. We use the Minimum Matching distance - a metric previously studied in the context of polytopes estimation \cite{nguyen2015posterior} to compare the quality of the fitted DSN model returned by a variety of inference algorithms. We defer additional details to Supplement \ref{sec:supp:experiment_details}.

In Fig. \ref{fig:sample_size} we see that VLAD and VLAD-$\alpha$ significantly outperform all baselines. Further, the estimation error reduces with increased sample size verifying statements of Theorems \ref{thm:consistency} and \ref{thm:alpha_consistency}. We note that Stan HMC may also achieve good performance, however it is very costly to fit (e.g., 40 HMC iterations for Poisson case and $n=30000$ took 14 hours compared to 7 seconds for VLAD), therefore we had to restrict number of iterations, which explains its wider error bars across experiments.

\textbf{Geometry of the DSN}
To study the role of geometry of the DSN we rescale extreme points towards their mean by uniform random factors $c_k \thicksim \text{Unif}(c_{\min},1)$ for $k=1,\ldots,K$ and vary $c_{\min}$ in Fig. \ref{fig:scale} (smaller values imply more severe skewness of the latent simplex). To isolate the effect of the geometry of the DSN, we compare to GDM combined with knowledge of true $\alpha$ and extension parameter estimation using Algorithm \ref{algo:extensionParams} (GDM-MC). If the underlying simplex is equilateral, GDM-MC will be equivalent to VLAD-$\alpha$.

In Fig. \ref{fig:scale} we see that VLAD and VLAD-$\alpha$ are robust to varying skewness of the DSN. On the contrary, GDM-MC is only accurate when the latent simplex becomes closer to equilateral. This experiment verifies geometric motivation of our work --- in practice we can not expect latent geometric structure to be necessarily equilateral and geometrically robust method such as VLAD is more reliable.


\textbf{Varying Dirichlet prior} To complete our simulation studies we verify $\alpha$ estimation procedure proposed in Section \ref{sec:alpha} and analyzed in Theorem \ref{thm:alpha_consistency}. It is also interesting to compare performance of other baselines for larger $\alpha$ --- scenario often overlooked in the literature.

In Fig. \ref{fig:alpha} (and in previous experiments) we see that performance gap between VLAD and VLAD-$\alpha$ is very small, supporting effectiveness of our $\alpha$ estimation procedure across probability kernels. Additionally, we see that higher values of $\alpha$ lead to degrading performance of all considered methods, however VLAD degrades more gracefully.

\subsection{Real Data Analysis}
\paragraph{Topic modeling}
We analyze a collection of news articles from the New York Times. After preprocessing, we have 5320 unique words and 100k training documents with 25k left out for perplexity evaluation. We also compare semantic coherence of the topics \cite{newman2010automatic}.

In Table \ref{table:data} (left) we present results for $K=80$ topics.
The Gibbs sampler has the best perplexity score, but it falls behind in topic coherence. VLAD estimated $\alpha=0.05$ and has approximately same perplexity and coherence as GDM, while being 5 times faster. VLAD identified contextually meaningful topics, as can be seen from good coherence score and by eye-balling the topics --- they cover a variety of concepts from fishing and cooking to the Enron scandal and cancer.
The top 20 words for each of the VLAD topics are provided along with the code.
\vspace{-1pc}

\paragraph{Stock market analysis}
We collect variations (closure minus opening price) for 3400 days and 55 companies. We train several algorithms on data from the first 3000 days and report the average distance between the data points from the last 400 days and fitted simplices (i.e., Frobenius norm). This metric alone might be misleading since stretching any simplex will always reduce the score, therefore we also report the volumes of corresponding simplices. Results are summarized in Table \ref{table:data} (right) --- our method (estimated $\alpha=0.05$) achieves comparable fit in terms of the Frobenius norm with a more compact simplex.
Among the factors identified by VLAD, we notice a growth component related to banks (e.g., Bank of America, Wells Fargo). Another factor suggests that the performance of fuel companies like Valero Energy and Chevron are inversely related to the performance of defense contractors (Boeing, Raytheon).
\vspace{-1pc}
\section{Summary and Discussion}
\label{sec:discussion}
The Dirichlet Simplex Nest model generalizes a number of popular models in machine learning applications, including LDA and several variants of non-negative matrix factorization (NMF). We also develop an algorithm that exploits the geometry of the DSN to perform fast and accurate inference. We demonstrate the superior statistical and computational properties of the algorithm on several real datasets and verify its accuracy through simulations.

One of the key distinctions between the DSN model and NMF models is we replace the separability assumption by a Dirichlet prior on the weights. The main benefit of this approach is it enables us to model data that does not contain archetypal points \cite{Cutler1994Archetypal}.
%
Among the limitations of our approach is the reliance on the Dirichlet distribution assumption in a crucial way, that the Dirichlet distribution is symmetric on the standard probability simplex $\Delta^{K-1}$. In theory, the algorithm breaks down when the Dirichlet distribution is asymmetric. Surprisingly, in simulations at least, we found that VLAD seems quite robust in recovering the correct direction of extreme points, even as most existing methods break down in such situations. These findings are reported in Supplement \ref{sec:supp:assymetric_alpha}.

%
%
\textbf{Acknowledgement} Support provided by NSF grants DMS-1830247, CAREER DMS-1351362, CNS-1409303 and a Margaret and Herman Sokol Faculty Award are gratefully acknowledged.

\clearpage
\bibliographystyle{icml2019}
\bibliography{MY_ref,yuekai}

\clearpage 

\appendix

\section{Proofs of Theorems}
In this section we present the proofs of main theorems in Section~\ref{sec:theory}. We will first reintroduce some notations for the reader's convenience.

{\bf Notation}
Let $\lambda_{\max}(A)$ and $\lambda_{min}(A)$ denote the largest and smallest non-zero singular values of the matrix $A$. We use $f(\cdot)$ to denote the density of $\mathbb{Q}$ with respect to Lebesgue measure on the $K-1$ dimensional subspace containing the simplex $\mathscr{B}$. Let $g(\cdot)$ be the density of $\mathbb{P}$ with respect to the Lebesgue measure on the $K-1$ dimensional space containing the eigenvectors of $\Sigma_{tot}^K$, where $\Sigma_{tot}^K $ is best $K-1$-rank approximation matrix of $\Sigma_{tot}:=BSB^T + \epsilon_0 I_D $ and $\epsilon_0 I_D $ is a uniform upper bound on $\Cov[x_i\mid\theta]$. Let $\Sigma$ be the population covariance matrix with $\Sigma^K$ as the best $K-1$ rank approximation. Note that
\begin{align}
    \Sigma & = \Cov(X_i)=\mathbb{E}[\Cov(X_i|\mu_i)]  + \Cov(\mathbb{E}[X_i|\mu_i]) \nonumber \\ & \leq \epsilon_0 I_D  + BSB^T.
    \end{align}
\subsection{Proof of Theorem 1}
\label{Proof of Theorem 1}

The following is a standard assumption to ensure the consistency of the $k$-means procedure embedded in our algorithm:
\begin{enumerate}
\item[(a.1)] Pollard's regularity criterion (PRC): The Hessian matrix of the function 
$c \mapsto  \mathbb{Q} \phi_{BSB^T}(\cdot,c)\text{ evaluated at }c^*$
for all optimizer $c^*$ of $\mathbb{Q} \phi_{BSB^T}(\cdot,c)$ is positive definite, with minimum eigenvalue $\lambda_0>0$.
\end{enumerate}

\begin{proof}
First, we note that under the assumption of the noiseless setting, by following along the lines of the proof of  Lemma~\ref{lemma:simplex_transform}, it can be seen that if $c^*=(c^*_1,\dots, c^*_K)$ optimize Eq.~\eqref{eq:clusteringUncenteredData} and $v_k$'s are such that $(v_1,\dots, v_K)$ form the empirical CVT centroids of $\Delta^{K-1}$, then $c^*_i=BPv_i + c_0$, where $c_0$ is the population centroid. 

Next, the convergence of the empirical CVT centroids to the  corresponding population CVT centroids occurs at rate $O_{\mathbb{P}}(\frac{1}{\sqrt{n}})$ rate following \citet{pollard1982CLT}. The consistency of the extreme points of the Dirichlet Simplex Nest follows by the continuous mapping theorem since  
\begin{eqnarray}
\frac{\|Pe_k\|_2}{\|Pv_k \|_2} =\frac{\|e_k - \frac1K\ones_K\|_2}{\|v_k - \frac1K\ones_K\|_2} = \frac{\|B(e_k - \frac1K\ones_K)\|_2}{\|B(v_k - \frac1K\ones_K)\|_2},
\end{eqnarray}
where $e_1,\dots,e_K$ are the canonical basis vectors on $\mathbb{R}^K$ denoting the vertices of $\Delta^{K-1}$. 

Finally, the knowledge of $\alpha$ enables us to compute $\frac{\|e_k - \frac1K\ones_K\|_2}{\|v_k - \frac1K\ones_K\|_2}$. This concludes the proof.
\end{proof}

\subsection{Proof of Theorem 2}
\label{sec:supp:theorem2_proof}
It is considerably more challenging to establish the error bounds for our algorithm in the general setting where the observations are noisy. First, let us define the following:
 \begin{align*}
\mathscr{C}_{\mathbb{P}_n}= &\{ c^*:c^*=\argmin_{c \in \mathbb{R}^{kD}} \mathbb{P}_n \phi_{(\Sigma_n)^K}(\cdot,c)\\
&=\argmin_{c \in \mathbb{R}^{kD}} \frac{1}{n}\sum_{i=1}^n \phi_{(\Sigma_n)^K}(\tilde{X}_i,c)\},\\
\mathscr{C}_{\mathbb{Q}}=&\{c^*:c^*=\argmin_{c \in \mathbb{R}^{kD}}  \mathbb{Q} \phi_{BSB^T}(\cdot,c)\}.
\end{align*}

Recall the following assumptions from the main text: 
\begin{enumerate}
\item[(a.2)]  The Hessian matrix of the function 
$c \mapsto  \mathbb{P} \phi_{(\Sigma)^K}(\cdot,c)$  evaluated at $c^*$
for all optimizer $c^*$ of $\mathbb{P} \phi_{(\Sigma)^K}(\cdot,c)$ is uniformly positive definite with minimum  eigenvalue bounded below from  some $\lambda_0>0$, for all $(\Sigma)^K$ such that $(\Sigma-BSB^T) \leq \tilde{\epsilon} I_D $, for some $\tilde{\epsilon}>0$.

 \item [(b)] There exists $\epsilon_0 > 0$ such that $ \epsilon_0 I_D-\conv(X|\theta)$ is positive semi-definite uniformly over $\theta \in \Delta^{K-1}$.
 \item [(c)]There exists $M_0$ such that for all $M> M_0$,
\begin{eqnarray}
  \int_{\mathcal{B}(\sqrt{M},c_0)^c}\|x-c_0\|^2_2 g(x) \mathrm{d}x \leq \frac{k_1}{M}, \nonumber
\end{eqnarray}
for some universal constant $k_1$, where $\mathcal{B}(\sqrt{M},c_0)$ is a ball of radius $\sqrt{M}$ around the population centroid, $c_0$.
\end{enumerate}
The assumptions (b) and (c) are very general assumptions and satisfied by a vast array of noise distributions, especially those with subexponential tails. In particular, the noise distributions considered in this work all satisfy these assumptions.

\begin{proof}
The proof proceeds by the following steps:

First, in \textbf{Step 1}, we show that it is enough to restrict attention to the population estimates instead of empirical estimates. Next, in \textbf{ Step 2}, we show that the k-means objectives for distributions of $\mu_i$'s and $x_i$'s are close. \textbf{Step 3} shows that the objective values at the respective minimizers are also close to each other for the distributions considered in \textbf{Step 2}. Finaly, \textbf{ Step 4} uses the strong convexity condition of (a.2) to bound the distance between respective k-means centers, and \textbf{Step 5} translates this bound to the estimation of the simplex vertices.

In that regard,
\paragraph{Step 1:} Following \citet{pollard1982CLT}, the empirical estimates of CVT centroids optimizing $\mathbb{P} \phi_{\Sigma^K}(\cdot,c)$ converges to the corresponding population estimate at rate $O_{\mathbb{P}_n}(n^{-1/2})$. Thus it is enough to restrict attention to the population estimates.
 \paragraph{Step 2:} We will show that for all $\epsilon_0$ sufficiently small,
\begin{equation}
|\mathbb{Q}\phi_{BSB^T}(\cdot,c) - \mathbb{P} \phi_{\Sigma^K}(\cdot,c)| = O(\epsilon_0^{1/3})  \nonumber
\end{equation}
uniformly over $c \in \mathscr{B}^K$.

Since $\mathbb{Q}$ denotes the distribution corresponding to $\mu_i$'s, this distribution places its entire mass inside the simplex, therefore all minimizers of the function  $\mathbb{Q} \phi_{BSB^T}(\cdot,c)$ lie inside $\mathscr{B}^K$. We can hence restrict our attention to $c \in \mathscr{B}^K$.  By assumption (b), we have $BSB^T  \leq \Sigma^K $. Thus, it is enough to establish a bound for  $|\mathbb{Q}\phi_{BSB^T}(\cdot,c) - \mathbb{P} \phi_{\Sigma^K}(\cdot,c)| \ \ \forall \  c \ \in \mathscr{B}^K$.
\begin{equation}
\label{objective_difference}
\begin{split}
|\mathbb{Q}\phi_{BSB^T}(\cdot,c) - \mathbb{P}\phi_{\Sigma^K}(\cdot,c)| \leq|\mathbb{P}\phi_{\Sigma^K}(\cdot,c) - \mathbb{Q} \phi_{\Sigma^K}(\cdot,c)| \\
 \hspace{80pt} + |\mathbb{Q}\phi_{BSB^T}(\cdot,c) - \mathbb{Q}\phi_{\Sigma^K}(\cdot,c)|.
\end{split}
\end{equation}

 \paragraph{Step 2.1:} Now, to bound  the second term on the right hand side of Eq. \eqref{objective_difference} we use, 
\begin{alignat*}{2}
 &|\mathbb{Q}\phi_{BSB^T}(\cdot,c) - \mathbb{Q} \phi_{\Sigma^K}(\cdot,c)| \\
 &\quad \leq  \int |\phi_{BSB^T}(x,c)-\phi_{\Sigma^K}(x,c)|f(x) \text{d}x \\
 &\quad \leq \lambda_{\max}([BSB^T]^{\dagger}-[\Sigma^K]^{\dagger})  \\
 &\quad \leq  \lambda_{\max}([BSB^T]^{\dagger}-[(BSB^T + \epsilon_0I_{D}^K]^{\dagger}) \\
 &\quad \leq \frac{\epsilon_0}{\lambda_{\min}(BSB^T) \lambda_{\min}(BSB^T + \epsilon_0I_{D}^K)},
\end{alignat*}
where $B^{\dagger}$ denotes the pseudo-inverse of $B$, and $I_D^K$ is the matrix with top $K-1$ diagonal elements as $1$, the rest zeros.

 \paragraph{Step 2.2:} Turning to the first term on right hand side of Eq. \eqref{objective_difference}, 
we note that $\|\beta_i -\beta_j\|^2 \leq \frac{K-1}{K} \lambda_{\max}(BSB^T)$. Therefore a compact ball of radius $ a\lambda_{\max}(BSB^T)$ around the centroid $c_0$  of the simplex $\mathscr{B}$ for all sufficiently large constants $a > \frac{K-1}{K}$ contains the simplex completely. Consider a ball $\mathcal{B}(\sqrt{M},c_0)$ of radius $\sqrt{M}$, with $M=a \lambda_{\max}(BSB^T)$ around the centroid $c_0$, the scalar $a$ to be chosen later. For any $M>0$,
\begin{align}
\label{boundedness-lipschitz}
|\mathbb{P}\phi_{\Sigma^K}(\cdot,c) - \mathbb{Q} \phi_{\Sigma^K}(\cdot,c)| \leq \biggr |\int_{\mathcal{B}(\sqrt{M},c_0)^c} \phi_{\Sigma^K}(x,c) g(x)\text{d}x\biggr | \nonumber \\ 
+ \biggr |\int_{\mathcal{B}(\sqrt{M},c_0)} \phi_{\Sigma^K}(x,c)  [ g(x)- f(x)]\text{d}x\biggr |.
\end{align}
 \paragraph{Step 2.2.1:} For the first term on the right hand side of Eq. \eqref{boundedness-lipschitz}, we see that,
\begin{equation}
\label{boundedness}
\begin{split}
&\int_{\mathcal{B}(\sqrt{M},c_0)^c} \phi_{\Sigma^K}(x,(c_1,\dots,c_K)) g(x)\text{d}x\\
&\leq \min_{i \in \{1,\dots, K\}} \int_{\mathcal{B}(\sqrt{M},c_0)^c} \|x-c_i\|^2_{\Sigma^K} g(x) \text{d}x \\
 &\leq \max 2\|c_i-c_0\|^2_2 \mathbb{P}(X \in \mathcal{B}(\sqrt{M},c_0)^c ) \\
&\hspace{50pt} + \frac{2}{\lambda_{\min}(BSB^T)}\int_{\mathcal{B}(\sqrt{M},c_0)^c} \|x-c_0\|^2_2 g(x) \text{d}x .
\end{split}
\end{equation}
The first inequality follows from Fatou's lemma, while the second follows from the fact that $\|a +b\|^2_2 \leq 2(\|a\|^2_2 + \|b\|_2^2)$. 

Suppose that the noise distribution is subexponential for all latent locations $\theta \in \mathscr{B}$. Combining this with the Chebyshev inequality and condition (c), Eq. \eqref{boundedness} can be re-written as:
\begin{equation}
\begin{split}
&\int_{\mathcal{B}(\sqrt{M},c_0)^c} \phi_{\Sigma^K}(x,(c_1,\dots,c_K)) g(x)\text{d}x \\
& \leq \tilde{C}\lambda_{\max}(BSB^T) \frac{Var(X)}{M} + \frac{2k_1 }{\lambda_{\min}(BSB^T)M}\\
& \leq \tilde{C} \frac{2(K-1)\lambda_{\max}^2(BSB^T)}{M} + \frac{2k_1 }{\lambda_{\min}(BSB^T)M}
\end{split}
\end{equation}
 for some universal constant $k_1$. 
\paragraph{Step 2.2.2:} For the second term on the right hand side on Eq. \eqref{boundedness-lipschitz}, we use the following result. 
\paragraph{Claim 1.} For $ M= a\lambda_{\max}(BSB^T)$, when centroids $c_i \in \mathscr{B} \ \ \forall \ \ i$, $\phi_{\Sigma^K}(x,c=(c_1,\dots,c_K))$ as a function of $x$ is Lipschitz on $\mathcal{B}(\sqrt{M},c_0)$, with Lipschitz constant $\frac{4 \sqrt{M}}{\lambda_{\min}(BSB^T)}$.

Now using the above result, we can easily extend $\phi_{\Sigma^K}(x,c=(c_1,\dots,c_K))$ to a Lipschitz function on the entire domain.
For the particular choice of $a$, 
\begin{equation}
\begin{split}
&\biggr|\int_{B(\sqrt{M},c_0)} \phi_{(\Sigma)^K}(x,c)  ( g(x) - f(x)) \text{d}x\biggr | \\ 
&\quad\leq \frac{2 \sqrt{a\lambda_{\max}(BSB^T)}}{\lambda_{\min}(BSB^T)} \sup_{\|l\|_{Lip} \leq 1} \biggr |\int l(x)  (g(x) - f(x))\text{d}x \biggr |  \\
&\quad\leq \frac{2 \sqrt{a\lambda_{\max}(BSB^T)}}{\lambda_{\min}(BSB^T)} W_1(g,f) \\
&\quad\leq \frac{2 \sqrt{a\lambda_{\max}(BSB^T)}}{\lambda_{\min}(BSB^T)} \sqrt{(K-1) \epsilon_0}.
\end{split}
\end{equation}
In the above, $\|l\|_{Lip}$ denotes the Lipschitz constant of the function $l(\cdot)$.
The second inequality in the above equation follows from Kantorovich-Rubinstein duality while for the last inequality, we use the definition of the Wasserstein distance and take $(X,\mu)$ as the coupling with densities $X \sim g $  and $\mu \sim f$ marginally (cf.~\citet{Villani2008Optimal}). Then, for any upper bound $M_1$ on the variance of $\|X-\mu\|_2$ , $W_2(g,f) \leq M_1$, and we use the fact that $\sqrt{(K-1)\epsilon_0}$ forms such an upper bound.

Now, for the noise level $\epsilon_0>0$ sufficiently small,   there exists $\epsilon>0$, which is dependent on $\epsilon_0$, such that the open interval $\biggr (C' \frac{(K-1)\lambda_{\max}^2(BSB^T)}{\epsilon}, \frac{\lambda_{\min}^2(BSB^T)}{{\lambda_{\max}(BSB^T) (K-1) \epsilon_0}}\epsilon ^2/16 \biggr )$ is non-empty for any fixed constant $C'$. Whenever $a$ is chosen in this range, $|\mathbb{Q}\phi_{BSB^T}(\cdot,c) - \mathbb{P} \phi_{\Sigma^K}(\cdot,c)| \leq \epsilon$. Note that we can choose $\epsilon =O(\epsilon_0^{1/3})$ and $a=O(\epsilon_0^{-1/3})$ to satisfy the above condition.


\paragraph{Step 3:} In this step, we show that objective function values for k-means corresponding to that of the population distributions of $x_i$'s and $\mu_i$'s are close. Notice that the bounds obtained in Step 2 are uniform over $c \in \mathscr{B}$. For ease of writing, we denote $R_q(c)=\mathbb{Q}\phi_{BSB^T}(\cdot,c)$ and $R_p(c)=\mathbb{P} \phi_{\Sigma^K}(\cdot,c)$. Also, let $\argmin R_p(c)=c_p$ and  $\argmin R_q(c)=c_q$.  Then, for $\epsilon_0$ sufficiently small, it follows from the discussion above that

\begingroup\makeatletter\def\f@size{9}\check@mathfonts
\begin{equation}
\begin{split}
&|R_q(c_p)-R_q(c_q)| \\
&=|R_q(c_p)-R_p(c_p) + R_p(c_q)-R_q(c_q) + R_p(c_p)- R_p(c_q) | \\
& \leq |R_q(c_p)-R_p(c_p) + R_p(c_q)-R_q(c_q) | = O(\epsilon_0^{1/3}).
\end{split}
\end{equation}
\endgroup

\paragraph{Step 4:}  In this step, we show that $\|\argmin_{c}\mathbb{P} \phi_{\Sigma^K}(\cdot,c)- \argmin_{c}\mathbb{Q}\phi_{BSB^T}(\cdot,c) \|_2 \rightarrow 0$  as  $\epsilon_0 \rightarrow 0$. The intuition behind this is that since the functions $\mathbb{Q}\phi_{BSB^T}(\cdot,c)$ and $R_p(c)=\mathbb{P} \phi_{\Sigma^K}(\cdot,c)$ are point-wise close, and their minimized values are also close to one another, therefore, the points of minima must also be close. 
By a standard strong convexity argument, employing condition (a.2), for $\epsilon_0$ sufficiently small, we get,
 \begingroup\makeatletter\def\f@size{7}\check@mathfonts
 \begin{eqnarray}
 \|\argmin_{c}\mathbb{P} \phi_{\Sigma^K}(\cdot,c)- \argmin_{c}\mathbb{Q}\phi_{BSB^T}(\cdot,c) \|_2 \nonumber \\
 = O\left(\sqrt{\epsilon_0^{1/3}/\lambda_0}\right). 
 \end{eqnarray}
 \endgroup

  \paragraph{Step 5 :} 
  Finally, the error bound for the simplex vertices follows from a continuous mapping theorem's argument in a similar manner to that of the proof for Theorem \ref{thm:noiselessConsistency}.
\end{proof}

\paragraph{Claim 1.} For $ M= a\lambda_{\max}(BSB^T)$, when centroids $c_i \in \mathscr{B} \ \ \forall \ \ i$, $\phi_{\Sigma^K}(x,c=(c_1,\dots,c_K))$ as a function of $x$ is Lipschitz on $\mathcal{B}(\sqrt{M},c_0)$, with Lipschitz constant $\frac{4 \sqrt{M}}{\lambda_{\min}(BSB^T)}$.

\begin{proof}[Proof of Claim 1.]
\begin{equation}
\begin{split}
& \frac{ |\phi_{\Sigma^K}(x,c={c_1,\dots,c_K})-\phi_{\Sigma^K}(y,c={c_1,\dots,c_K})|}{\|x-y\|} \\
& \leq \max_{i \in \{1,\dots,K\}}\frac{| \|x - c_i\|_{\Sigma^K} - \|y- c_i\|_{\Sigma^K} |}{\|x-y\|_2} \\
& \leq \sup \frac{2\|x-y\|}{\lambda_{\min}(BSB^T)} \leq \frac{4 \sqrt{M}}{\lambda_{\min}(BSB^T)}.
\end{split}
\end{equation}
\end{proof}

\subsection{Consistent estimation of concentration parameter }
\label{consistent estimator}
In this section we first provide several easy calculations required for the estimating equations for some commonly used noise distributions.

\begin{lemma}
Depending on the data generating distribution, the covariance matrix of the DSN model is given as follows.
\begin{enumerate}
    \item[(a)] Gaussian data: $\Sigma= BS(\alpha)B^{T} + \sigma^2I_d$, provided that $x_i|\mu_i \thicksim \mathcal{N}(\mu_i,\sigma^2I_D)$.
    \item[(b)] Poisson data: $\Sigma= BS(\alpha)B^{T} + \text{Diag}(\sum_i B_i/K)$, provided that $x_{ij}|\mu_i \overset{ind}{\thicksim} Poi(\mu_{ij})$, where $B_i$ denotes the $i^{th}$ column of $B$ and $\text{Diag}(a)$ is a diagonal matrix with the $i^{th}$ diagonal element denoting the $i^{th}$ element of the vector $a$. Here, $\mu_i=(\mu_{i1},\ldots,\mu_{iD})$. 
    \item[(c)] Multinomial data: $\Sigma= (1-\frac{1}{N})BS(\alpha)B^{T} + \frac{1}{N}\text{Diag}(\sum_i B_i/K) - \frac{1}{N}(\sum_i B_i/K)(\sum_i B_i/K)^{T}$, provided that $x_i|\mu_i \thicksim \text{Multinomial}(N,\mu_{i1},\mu_{iD})$. Here, $\mu_i=(\mu_{i1},\ldots,\mu_{iD})$ is a probability vector. ($N$ resembles the number of words per document in the LDA model).
\end{enumerate}
\end{lemma}
\begin{proof}
 We compute $Cov(x_i)$ for each of the models. Note that $Cov(X_i)=\mathbb{E}(Cov(x_i|\mu_i)) + Cov(\mathbb{E}(x_i|\mu_i))$ from the tower property of conditional covariance, and $Cov(\mathbb{E}(x_i|\mu_i))= BS(\alpha)B^{T}$ for all the models. Therefore we just need the computation for $\mathbb{E}(Cov(x_i|\mu_i))$ for each of the models.

For the Gaussian model,  $\mathbb{E}(Cov(x_i|\mu_i)) =\sigma^2 I_D$.
   
 For the Poisson model, $\mathbb{E}(Cov(x_i|\mu_i)) = \mathbb{E}(\mu_i)=B\mathbb{E}(\theta_i)=\text{Diag}(\sum_i B_i/K)$, where the second equality follows as $\mu_i=B\theta_i$ by the model, and the last equality follows because $\theta_i \thicksim \text{Dir}(\alpha)$.
 
 For the multinomial model, $\mathbb{E}(Cov(x_i|\mu_i)) = \frac{1}{N}\mathbb{E}(\text{Diag}(\mu_i))- \frac{1}{N}Cov(\mu_i\mu_i^{T})=\frac{1}{N} (\text{Diag}(\sum_i B_i/K)-B S(\alpha)B^{T})$ from which the result follows.

\end{proof}

Equation~\ref{eqn:thm consistency}, for estimating $\alpha$ uses the data covariance matrix, $\hat{\Sigma}_n$. While this gives the correct estimating equation in the noiseless scenario, but for the noisy version we need to use $\tilde{\Sigma}_n$ instead where $\tilde{\Sigma}_n$ is a consistent estimator for $BS(\alpha)B^{T}$. The estimator estimator for different noise distributions can be obtained via the above lemma.

\subsubsection{Proof of Theorem 3}
\label{sec:supp:theorem3_proof}
The proof of consistency of the proposed estimate for the Dirichlet concentration parameter is given as follows.

\begin{proof}
 Notice that $\|\tilde{\Sigma}_n-BS(\alpha_0)B^T\| = o_P(1)$. Also, $ \|\hat{B}(\gamma(\alpha)) -B(\gamma(\alpha))\|=O_P(n^{-1/2})$ for all $\alpha \in \mathscr{C}$. Therefore 
 $ \| \hat{B}(\gamma(\alpha))S(\alpha)\hat{B}(\gamma(\alpha))^T - {B}(\gamma(\alpha))S(\alpha){B}(\gamma(\alpha))^T \| = O_P(n^{-1})$ for all $\alpha \in \mathscr{C}$. By monotonicity of the function $\varphi$, $BS(\alpha_0)B^T-{B}(\gamma(\alpha))S(\alpha){B}(\gamma(\alpha))^T$ as a function of $\alpha$ is injective for all $\alpha \in \mathscr{C}$. Therefore, $ \| \hat{B}(\gamma(\alpha_0))S(\alpha_0)\hat{B}(\gamma(\alpha_0))^T - \tilde{\Sigma}_n\| =o_P(1)$, by triangle inequality. The statement of the theorem then follows by employing a subsequence argument.
\end{proof}


\subsubsection{Identifiability of the concentration parameter}
\label{sec:supp:identif}
In the statement of Theorem 3, we require a condition which amounts to a identifiability condition of the parameter $\alpha$. 
In this section, we provide empirical evidence that the DSN model with unknown concentration parameter $\alpha$ is identifiable from second moments. 

As we shall see, the identifiability of $\alpha$ boils to the invertibility of a scalar function. Recall the covariance matrix of a $\Dir(\alpha)$ distribution is
\[
S(\alpha) = \frac{I_K-P_K}{K(K\alpha + 1)},
\]
where $P_K = \frac1K\ones_K\ones_K^T$ is the projector onto $\lspan\{\ones_K\}$. Let $B(\gamma) = \gamma(C - \mu) + \mu$ be the $\gamma$-extension of the (scaled) $K$-means centroids $C$ from the center of the DSN $\mu = \frac1KB\ones_K$. The question of the identifiability of the concentration parameter boils down to whether there are distinct $\alpha_1$ and $\alpha_2$ such that 
\begin{equation}
\begin{aligned}
&B(\gamma(\alpha_1))S(\alpha_1)B(\gamma(\alpha_1))^T \\
&\quad= B(\gamma(\alpha_2))S(\alpha_2)B(\gamma(\alpha_2))^T,
\end{aligned}
\label{eq:nonidentifiability}
\end{equation}
where $\gamma(\alpha)$ is the extension parameter that corresponds to concentration parameter $\alpha$. As long as $C$ has full column rank, we may pre and post-multiply \eqref{eq:nonidentifiability} by $C^\dagger$ and $(C^\dagger)^T$ respectively to see that \eqref{eq:nonidentifiability} is equivalent to 
\[
\begin{aligned}
&(\gamma(\alpha_1)(I_K - P_K) + P_K)S(\alpha_1)(\gamma(\alpha_1)(I_K - P_K) + P_K) \\
&\quad= (\gamma(\alpha_2)(I_K - P_K) + P_K)S(\alpha_2)(\gamma(\alpha_2)(I_K - P_K) + P_K).
\end{aligned}
\]
Recalling $S(\alpha)$ is a scalar multiple of $I_K - \frac1K\ones_K\ones_K^T$, we see that \eqref{eq:nonidentifiability} is equivalent to whether there are distinct $\alpha_1$ and $\alpha_2$ such that 
\[
\frac{\gamma(\alpha_1)^2}{K(K\alpha_1 + 1)} = \frac{\gamma(\alpha_2)^2}{K(K\alpha_2 + 1)}.
\]
This is equivalent to the invertibility of the function 
\begin{equation}
\varphi(\alpha) = \frac{\gamma(\alpha)^2}{K(K\alpha + 1)}.
\label{eq:scalarIdentifiability}
\end{equation}
Figure \ref{fig:identifiability} shows this function for $K = 10$ over a range of reasonable values of $\alpha$. We see that the function is in fact invertible.

\begin{figure}[h]
\centerline{\includegraphics[width=\columnwidth]{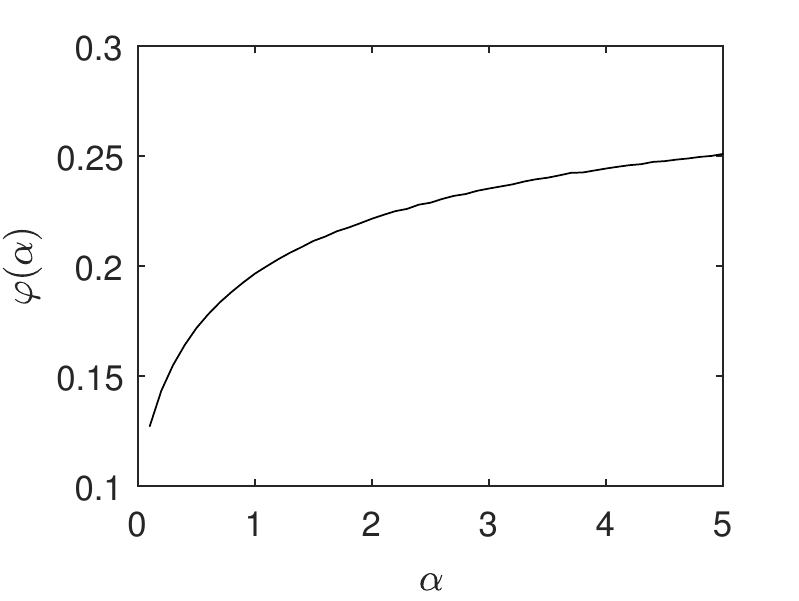}}
\caption{Empirical study of $\alpha$ identifiability.}
\label{fig:identifiability}
\end{figure}

Although Figure \ref{fig:identifiability} suggests \eqref{eq:scalarIdentifiability} is invertible, we do not have a rigorous proof. The main challenge is obtaining precise control on the growth of \eqref{eq:nonidentifiability}. Inspecting Figure \ref{fig:identifiability} shows that $\varphi(\alpha)$ is almost flat as soon as $\alpha$ exceeds $\frac52$. Intuitively, this is a consequence of the hardness of distinguishing between DSNs with large $\alpha$'s (and correspondingly large extension parameters). Mathematically, it is hard to obtain precise control on the growth of $\varphi(\alpha)$ because it is not possible to evaluate $\gamma(\alpha)$ explicitly. Although it is possible to show that
\begin{equation}
\gamma(\alpha) = \frac{1 - \frac1K}{\int_{V_k}e_k^T\theta p_\alpha(x)dx - \frac1K},
\label{eq:extPars}
\end{equation}
where $V_k = \{\theta\in\Delta^{K-1}:\argmax\{\theta_l:l\in[K]\} = k\}$ is the $k$-th Voronoi cell in a centroidal Voronoi tessellation of $\Delta^{K-1}$, $e_k$ is the $k^{th}$ canonical basis vector and $p_\alpha$ is the $\Dir(\alpha)$ density, it is hard to evaluate the integral. We defer an investigation of the identifiability of the concentration parameter to future work.

\section{Experimental Details}
\label{sec:supp:experiment_details}

\subsection{Computational cost of VLAD}

In this section, we tally up the computational cost of VLAD. The dominant cost it that of computing the top $K$ singular factors of the centered data matrix $\bar{X}$. This costs $O(DKn)$ floating point operations (FLOP's). The cost of the subsequent clustering step is asymptotically negligible compared to the cost of the SVD. Assuming each step of the $K$-means algorithm costs $O(Kn)$ FLOP's and the algorithm converges linearly, we see that the cost of obtaining an $O(\frac1n)$-suboptimal solution is $O(Kn\log n)$. We discount the cost of Monte Carlo estimates of the extension parameter because it can be tabulated. Thus the computational cost of the algorithm is dominated by the cost of computing the SVD.

\begin{figure*}[t]
\vskip -0.1in
\begin{subfigure}{.33\textwidth}
  \centering
  \includegraphics[width=\linewidth]{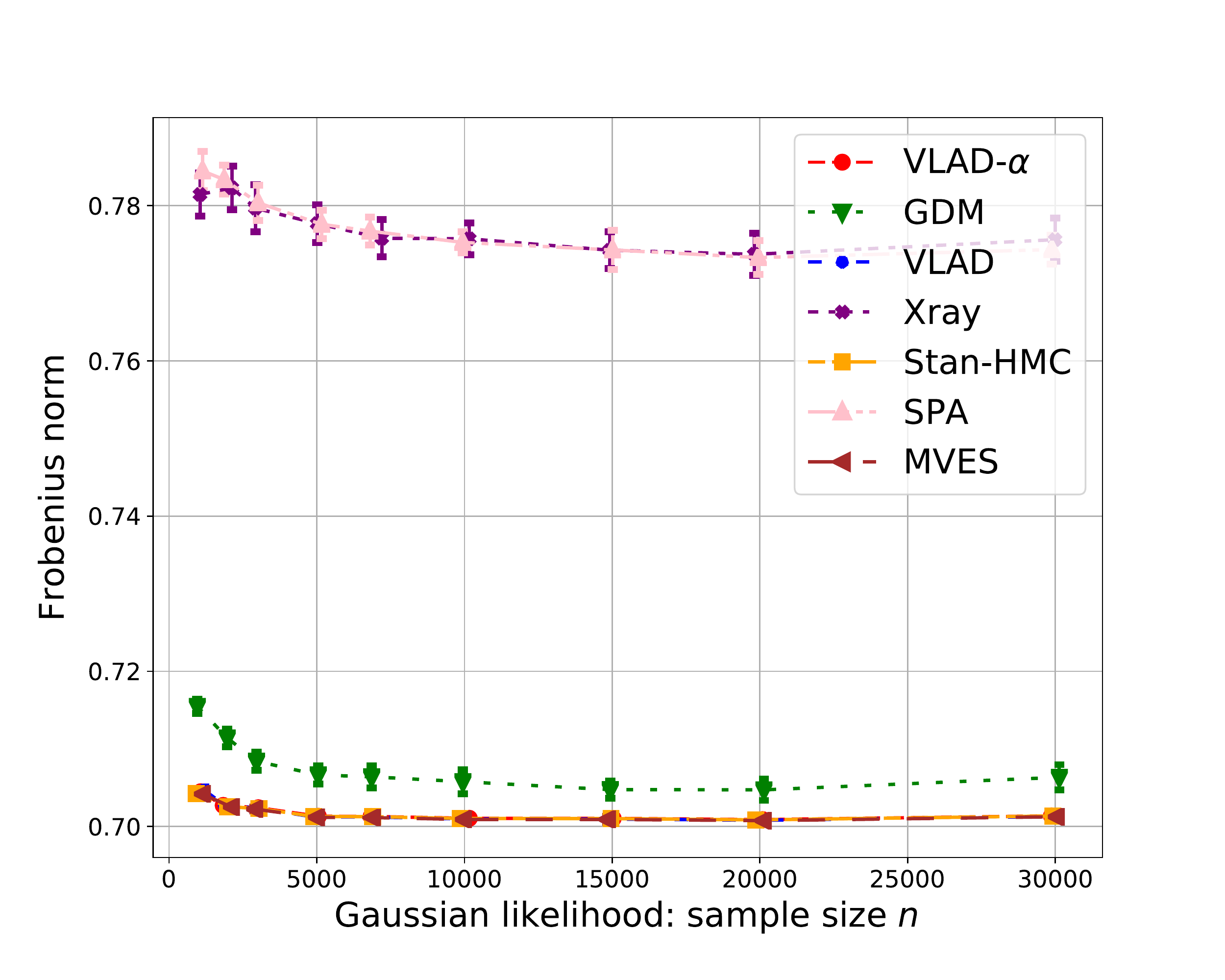}
  \vskip -0.1in
  \caption{Frobenius norm for Gaussian data}
  \label{fig:sample_size_gaus_ll}
\end{subfigure}
\begin{subfigure}{.33\textwidth}
  \centering
  \includegraphics[width=\linewidth]{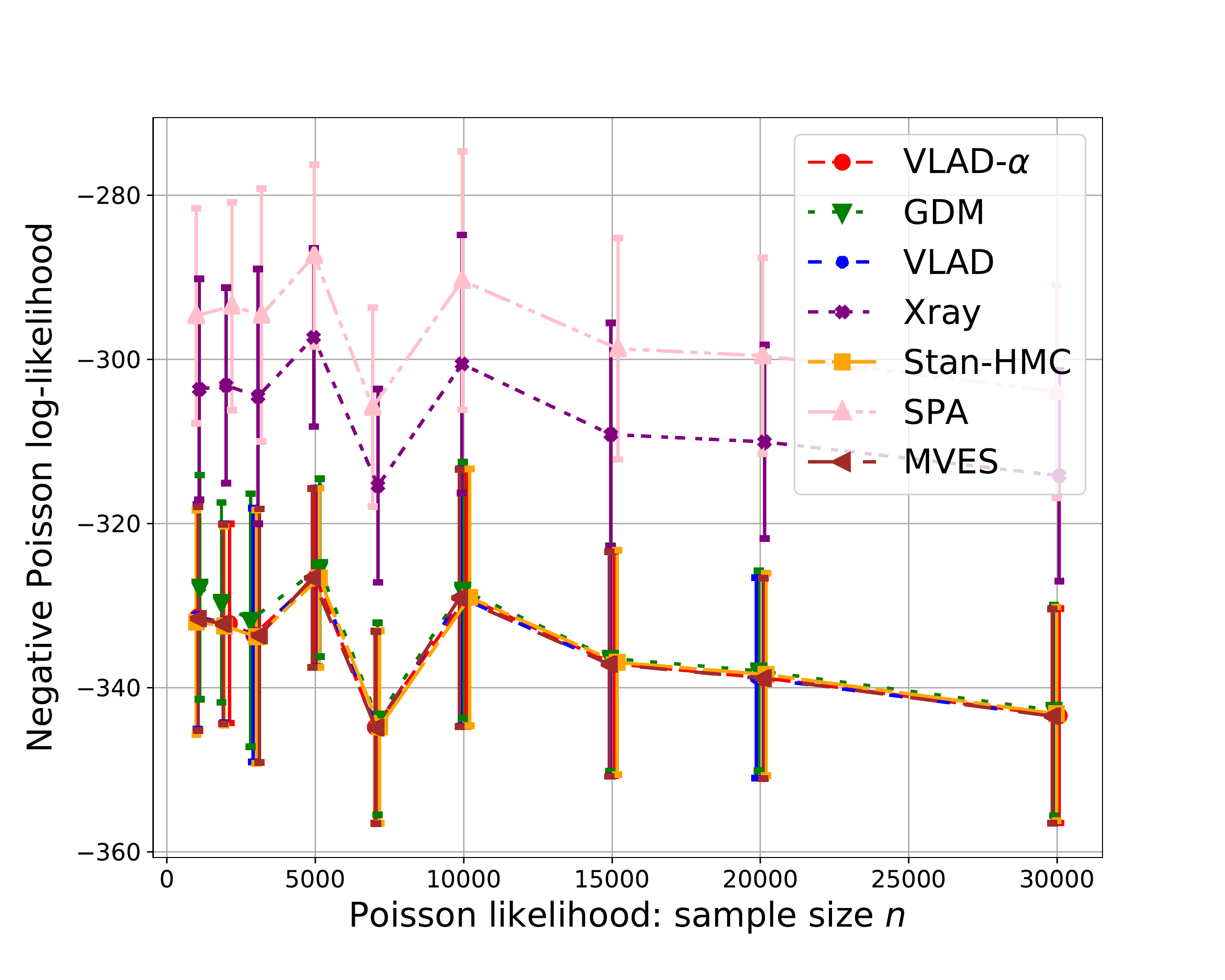}
  \vskip -0.1in
  \caption{Negative log-likelihood for Poison data}
  \label{fig:sample_size_pois_ll}
\end{subfigure}
\begin{subfigure}{.33\textwidth}
  \centering
  \includegraphics[width=\linewidth]{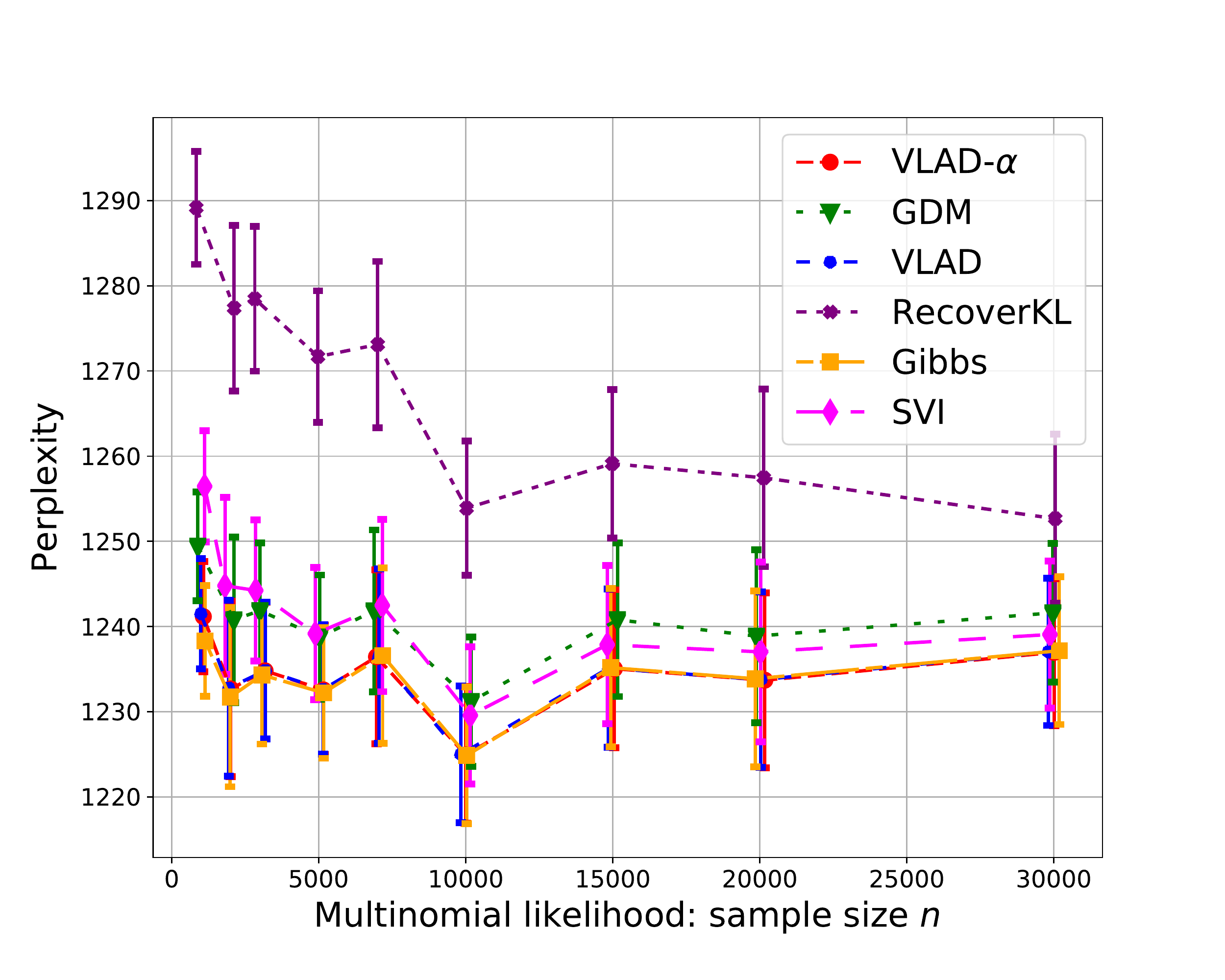}
  \vskip -0.1in
  \caption{Perplexity for LDA data}
  \label{fig:sample_size_lda_ll}
\end{subfigure}
\vskip -0.1in
\caption{Held out data performance for increasing sample size $n$}
\label{fig:sample_size_ll}
\end{figure*}

\begin{figure*}[t]
\vskip -0.1in
\begin{subfigure}{.33\textwidth}
  \centering
  \includegraphics[width=\linewidth]{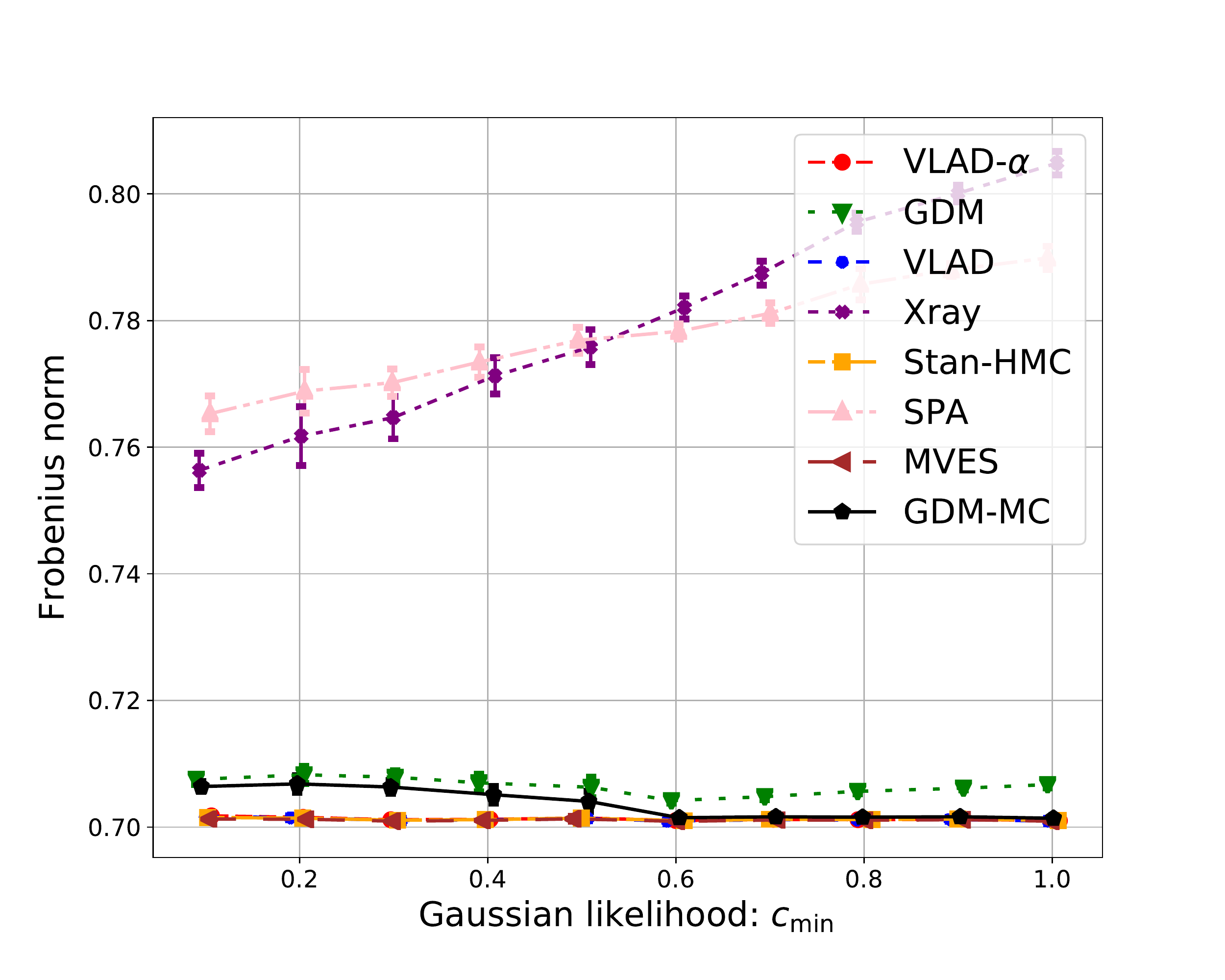}
  \vskip -0.1in
  \caption{Frobenius norm for Gaussian data}
  \label{fig:scale_gaus_ll}
\end{subfigure}
\begin{subfigure}{.33\textwidth}
  \centering
  \includegraphics[width=\linewidth]{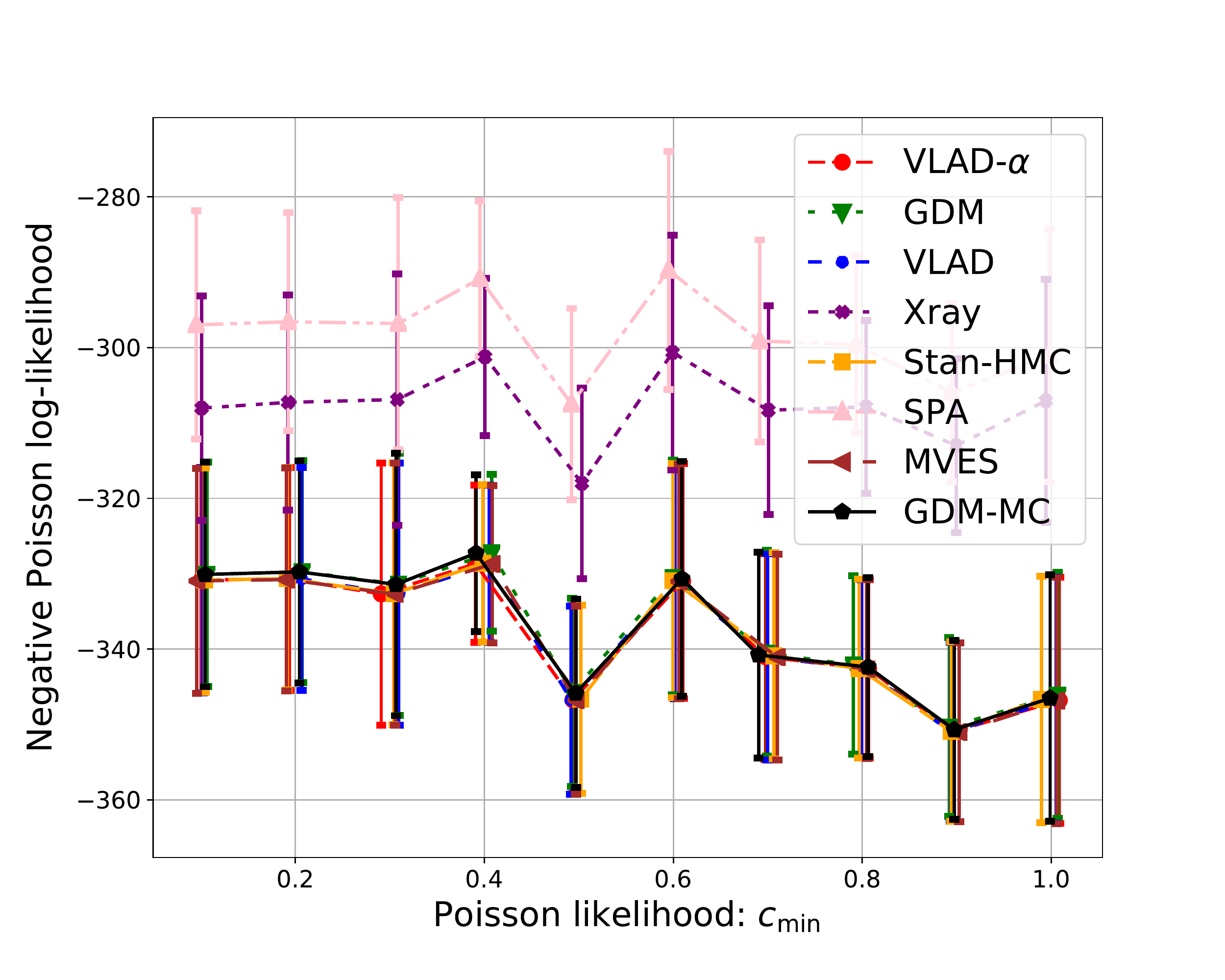}
  \vskip -0.1in
  \caption{Negative log-likelihood for Poison data}
  \label{fig:scale_pois_ll}
\end{subfigure}
\begin{subfigure}{.33\textwidth}
  \centering
  \includegraphics[width=\linewidth]{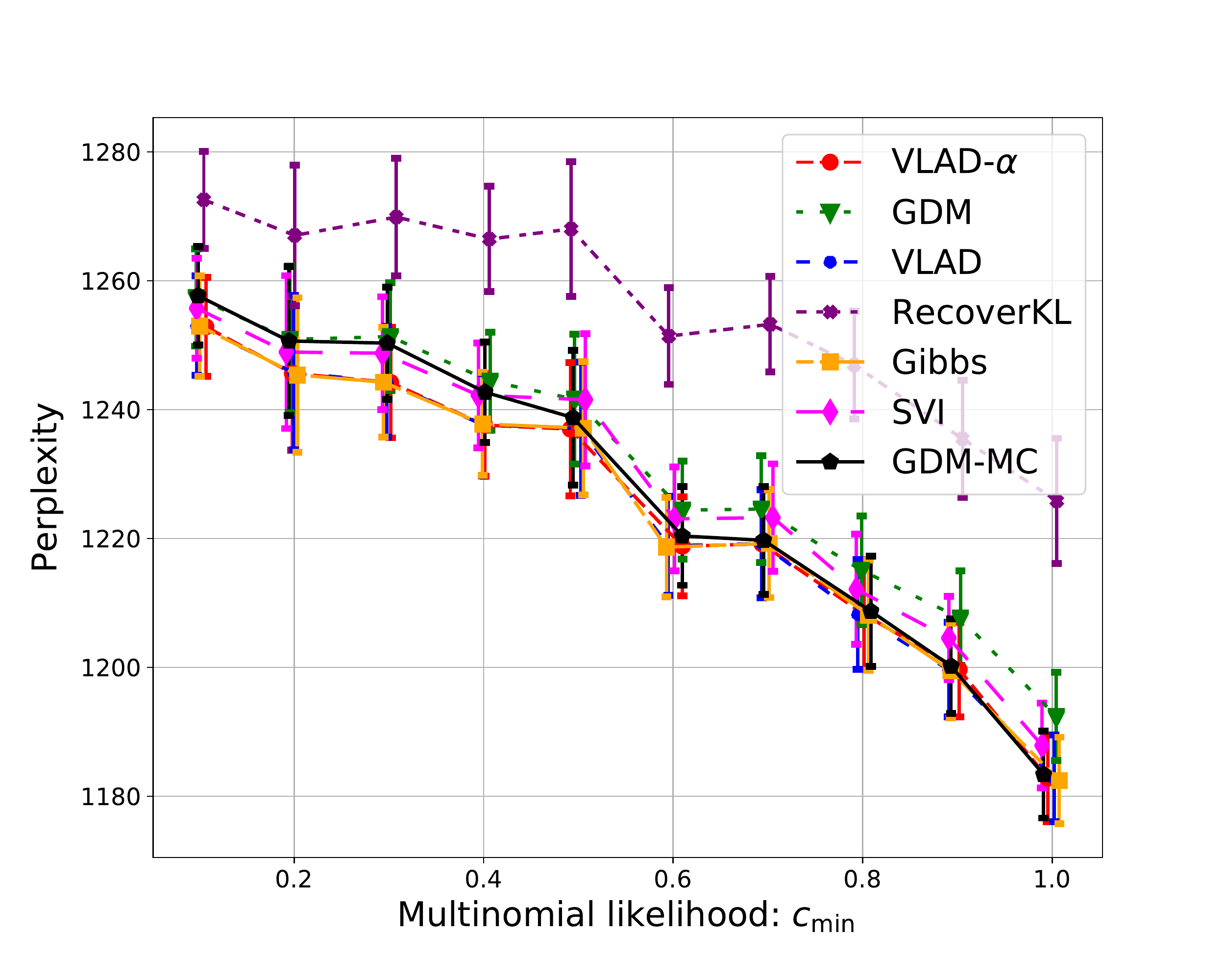}
  \vskip -0.1in
  \caption{Perplexity for LDA data}
  \label{fig:scale_lda_ll}
\end{subfigure}
\vskip -0.1in
\caption{Held out data performance for varying DSN geometry}
\label{fig:scale_ll}
\end{figure*}

\begin{figure*}[t]
\vskip -0.1in
\begin{subfigure}{.33\textwidth}
  \centering
  \includegraphics[width=\linewidth]{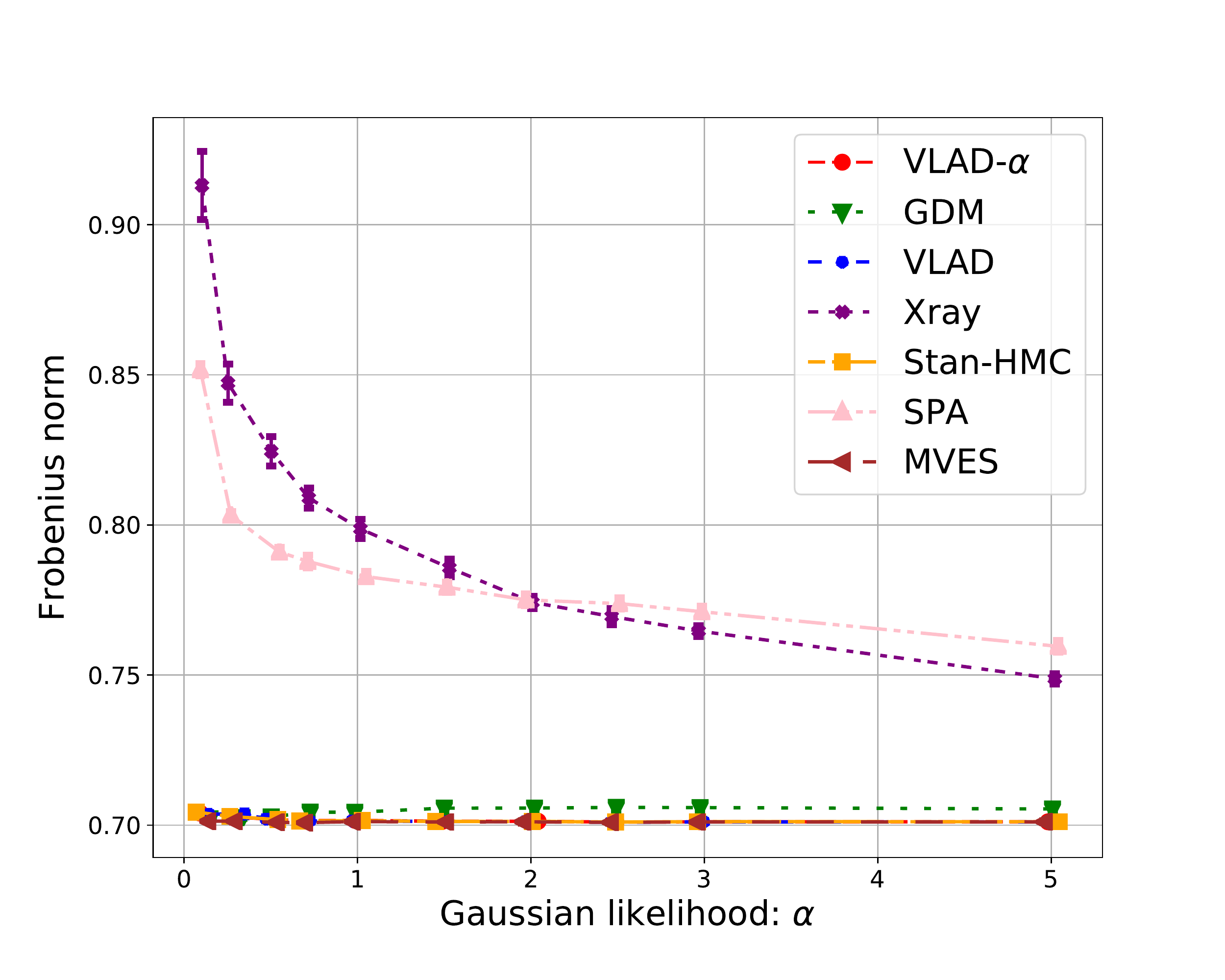}
  \vskip -0.1in
  \caption{Frobenius norm for Gaussian data}
  \label{fig:alpha_gaus_ll}
\end{subfigure}
\begin{subfigure}{.33\textwidth}
  \centering
  \includegraphics[width=\linewidth]{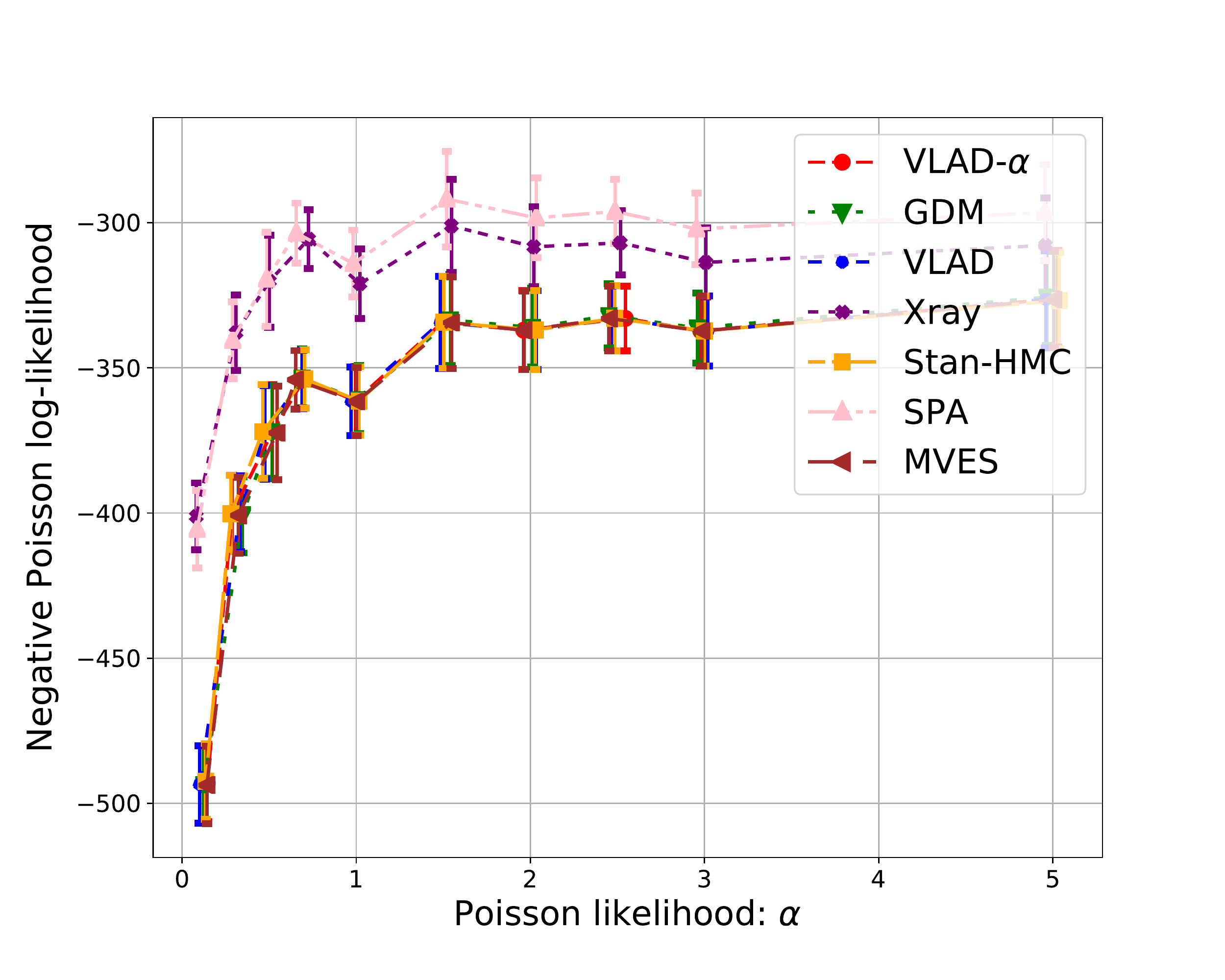}
  \vskip -0.1in
  \caption{Negative log-likelihood for Poison data}
  \label{fig:alpha_pois_ll}
\end{subfigure}
\begin{subfigure}{.33\textwidth}
  \centering
  \includegraphics[width=\linewidth]{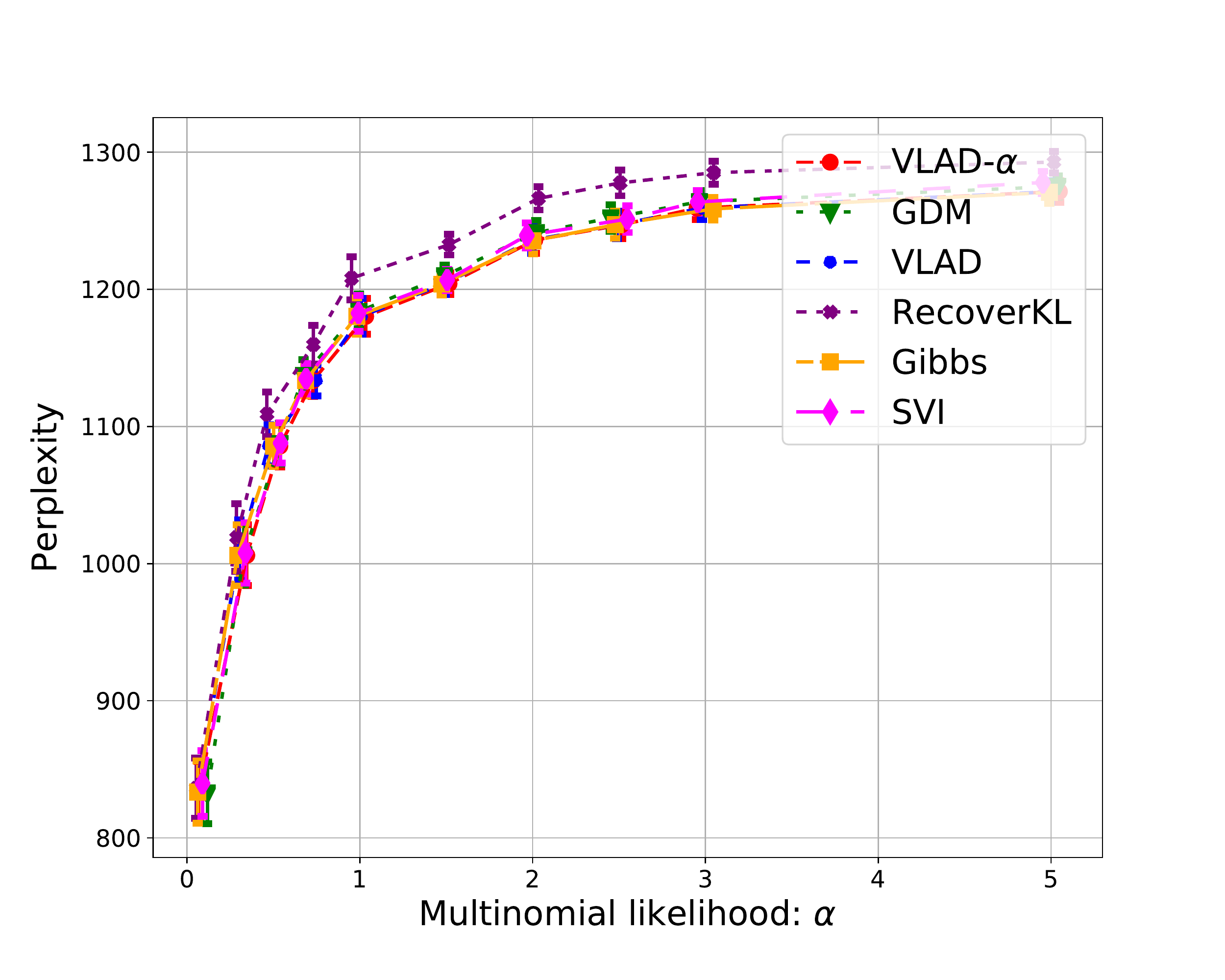}
  \vskip -0.1in
  \caption{Perplexity for LDA data}
  \label{fig:alpha_lda_ll}
\end{subfigure}
\vskip -0.1in
\caption{Held out data performance for increasing $\alpha$}
\label{fig:alpha_ll}
\end{figure*}

\subsection{Additional results}

\paragraph{Additional results for convergence behavior}
We complement the results presented in Fig.~\ref{fig:sample_size}  with the corresponding plots of the likelihood evaluated on a set of held out data. These results are summarized in Fig. \ref{fig:sample_size}. For all plots, the smaller value is better. We see that VLAD shows performance as good as HMC and Gibbs sampler at a much lower computational time. 

\paragraph{Additional results for geometry of the DSN}
Again, we further support our results of Fig.~\ref{fig:scale} with the corresponding held out data likelihood scores. Fig. \ref{fig:scale_ll} summarizes the results - VLAD shows competitive performance.

\paragraph{Additional results for varying Dirichlet prior}
In Figure \ref{fig:alpha_ll} we demonstrate held out data likelihood corresponding to experiments of Fig. ~\ref{fig:alpha}. We see that VLAD performs well in the whole range of analyzed values and likelihood kernels.

\paragraph{Data generation for simulations studies} For all experiments, unless otherwise specified, we set $D = 500, K=10, \alpha=2, n=10000$ (for LDA vocabulary size $D=2000$). To generate DSN extreme points, for Gaussian data we sample $\beta_1,\ldots,\beta_K \thicksim \mathcal{N}(0,K)$; for Poisson data $\beta_1,\ldots,\beta_K \thicksim \text{Gamma}(\ones,K\ones)$; for the LDA $\beta_1,\ldots,\beta_K \thicksim \text{Dir}_D(\eta)$ with $\eta=0.1$. To ensure skewed geometry we further rescale extreme points towards their mean by uniform random factors between 0.5 and 1. To do so first compute the mean of extreme points $C=\frac{1}{K}\sum_k \beta_k$ and then rescale each one with $\beta_k = C + c_k(\beta_k - C)$, where $c_k \thicksim \text{Unif}(c_{\text{min}},1)$. Except for the DSN geometry experiment, we set $c_{\text{min}}=0.5$.

Then we sample weights $\theta_i \thicksim \text{Dir}_K(\alpha)$ and data mean $\mu_i = \sum_k \theta_{ik} \beta_k$. For Gaussian data $x_i | \mu_i \thicksim \mathcal{N}(\mu_i, \sigma^2 I_D)$, $\sigma=1$; for Poisson data $x_i | \mu_i \thicksim \text{Pois}(\mu_i)$; for LDA we follow standard generating process \cite{blei2003latent} with 3000 words per document. All experiments were run for 20 repetitions and mean was used in the plots along with half standard deviation error bars.

\paragraph{Baseline methods and algorithms setups}
We considered four separability based NMF algorithms: Xray \cite{kumar2013fast} with code from \url{https://github.com/arbenson/mrnmf}; MVES \cite{chan2009convex} with code from \url{http://www.ee.nthu.edu.tw/cychi/source_code_download-e.php}; Sequential Projection Algorithm \cite{Gillis2014Fast} that we implemented in Python; RecoverKL \cite{arora2013practical} for the LDA case with code from \url{https://github.com/MyHumbleSelf/anchor-baggage}.

Bayesian NMF approaches often assume positive weights without the simplex constraint imposed by the Dirichlet prior on weights. Incorporating the simplex constraint complicates the inference \cite{paisley2014bayesian} as Dirichlet distribution is not conjugate to popular choices of data likelihood such as Gaussian or Poisson. Therefore we are not aware of any implementation for DSN type of models outside of the LDA scenario. We instead chose to compare to automated Bayesian inference methods. We implemented DSN inference with Poison and Gaussian likelihoods in Stan \cite{carpenter2017stan} and considered all three supported estimation procedures: HMC with No U-Turn Sampler \cite{hoffman2014no}, MAP optimization and \cite{kucukelbir2017automatic} Automatic Differentiation Variational Inference. MAP optimization and ADVI performed poorly and we did not report their performance. HMC was always trained with true value of $\alpha$ and with knowledge of $\sigma=1$ for the Gaussian scenario. Number of iterations was set to 80 for $n<3000$, 60 for $n=3000$ and 40 for $n>3000$. We had to restrict number of iterations due to prohibitively long running time (40 iterations for $n=30000$ took 3.5 hours for Gaussian likelihood and 14 hours for Poisson likelihood; VLAD took 7 seconds in both cases). For the LDA, we used Gibbs sampler \cite{griffiths2004finding} from \url{https://github.com/lda-project/lda} trained for 1000 iterations (1000 iterations for $n=30000$ took 3.6 hours; VLAD took 3min). Gibbs sampler was trained with true values of $\alpha$ and $\eta$. We used Stochastic Variational Inference \cite{hoffman2013stochastic} implementation from scikit-learn \cite{pedregosa2011scikit} and trained it with true values of $\alpha$ and $\eta$.

For the Geometric Dirichlet Means \cite{yurochkin2016geometric} we used implementation from \url{https://github.com/moonfolk/Geometric-Topic-Modeling} with 8 $K$-means restarts and ++ initialization.

VLAD was implemented in Python using numpy SVD package and scikit-learn \cite{pedregosa2011scikit} $K$-means clustering with 8 restarts and ++ initialization. The code is available at \url{https://github.com/moonfolk/VLAD}.

For the NYT data \url{https://archive.ics.uci.edu/ml/datasets/bag+of+words} we trained Gibbs sampler with $\alpha=0.1$ and $\eta=0.1$ for 1000 iterations and SVI with default settings. For the stock data we trained HMC for 100 iterations with $\alpha=0.05$.

\section{On asymmetric Dirichlet prior}
\label{sec:supp:assymetric_alpha}

In our work we assumed that $\theta_i \thicksim \text{Dir}_K(\alpha)$, where $\alpha \in \mathbb{R}_+$. When $\alpha$ is a scalar, the corresponding Dirichlet distribution is referred to as symmetric. More generally, $\alpha \in \mathbb{R}_{+}^K$ is a vector of parameters. Our algorithmic guarantees, such as alignment of CVT centroids of $\Bscr$, extreme points and centroid of $\Bscr$ and equivalence of extension parameters for all extreme points directions, fail for the general asymmetric case. \citet{wallach2009rethinking} showed that more careful treatment of the parameter $\alpha$ can improve the quality of the LDA topics. Geometric treatment of the asymmetric Dirichlet distribution remains to be the question of future studies. To facilitate the discussion, here we visualize the problem using toy $D=3,K=3$ example (similar to Fig.~\ref{fig:triangle} ) with $\alpha=(0.5,1.5,2.5)$. Results of the four different algorithms are shown in Fig. \ref{fig:triangle_alpha}. Note that for VLAD (Fig. \ref{fig:gam_alpha}) we only show the directions of the line segments of the obtained sample CVT centroids and the data center, since we do not have a procedure for extension parameter estimation in the asymmetric Dirichlet case. We see that all of the algorithms fail with various degrees of error and notice that the directions obtained by VLAD no longer appear consistent, however do not deviate drastically from the truth. We propose to call such toy triangle experiment a \emph{triangle test} and hope to "pass" the asymmetric Dirichlet \emph{triangle test} in the future work.

\begin{figure*}[t]
\vskip -0.1in
\begin{subfigure}{.5\textwidth}
  \centering
  \includegraphics[width=\linewidth]{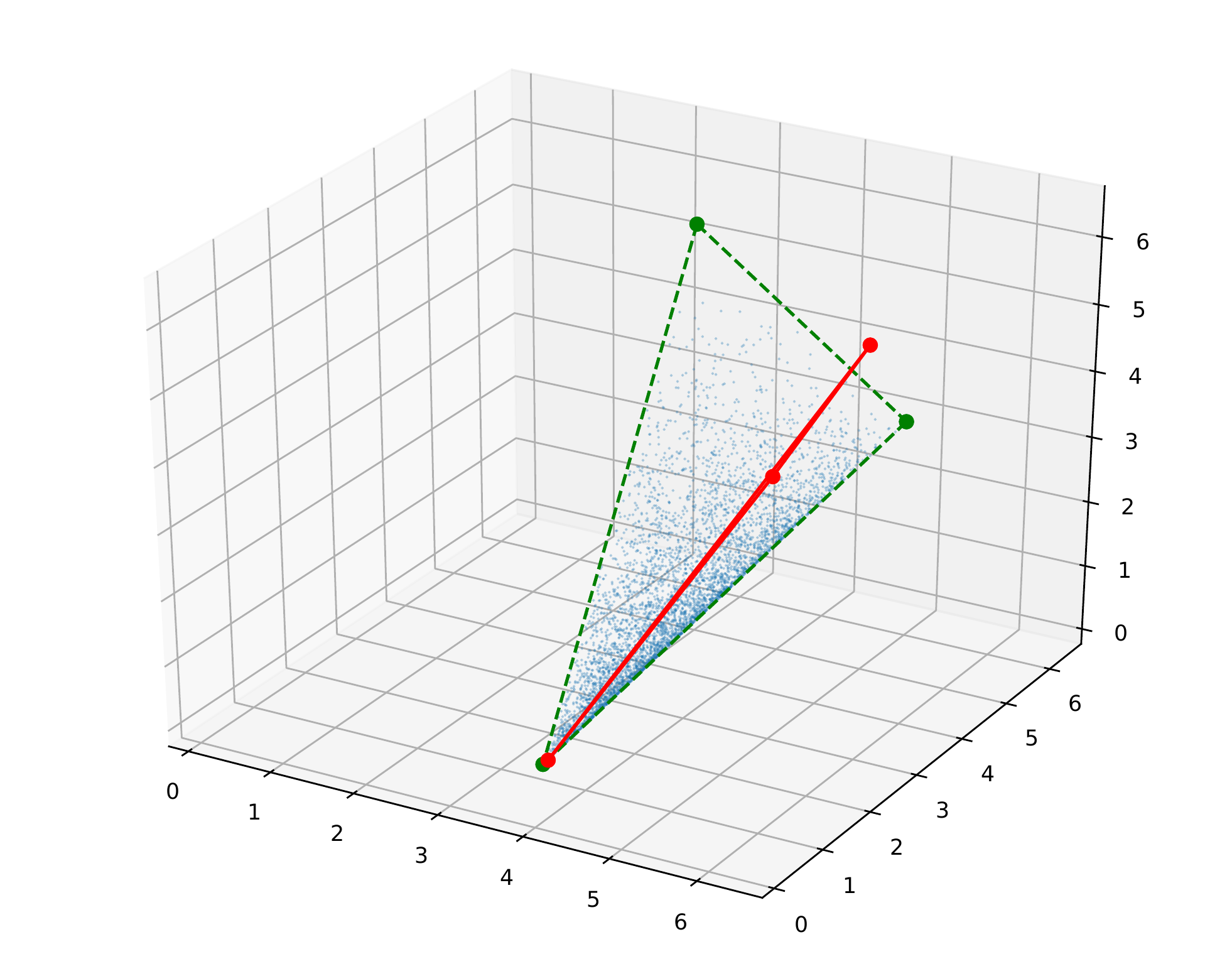}
  \caption{GDM}
  \label{fig:gdm_alpha}
\end{subfigure}
\begin{subfigure}{.5\textwidth}
  \centering
  \includegraphics[width=\linewidth]{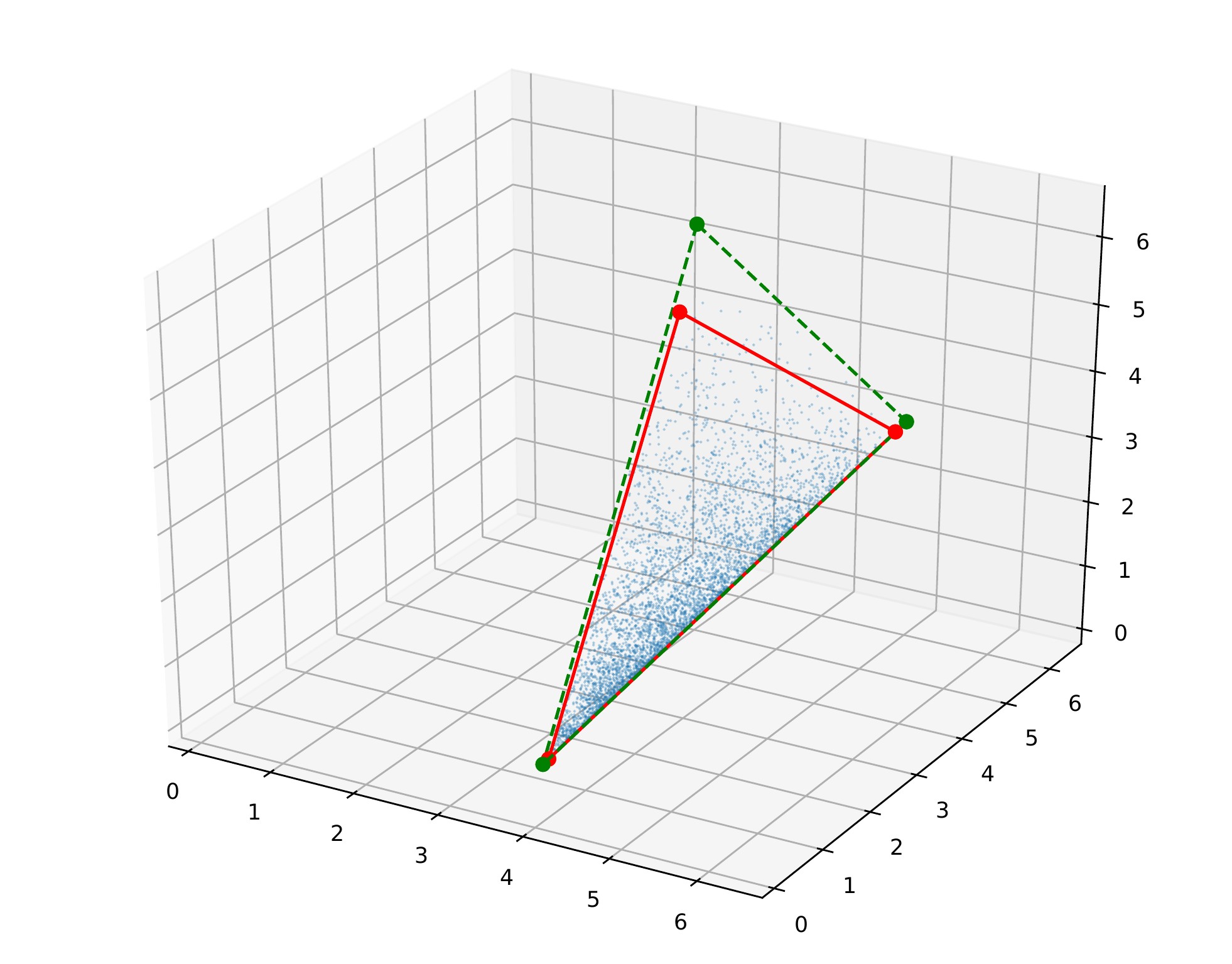}
  \caption{Xray}
  \label{fig:xray_alpha}
\end{subfigure} \\
\begin{subfigure}{.5\textwidth}
  \centering
  \includegraphics[width=\linewidth]{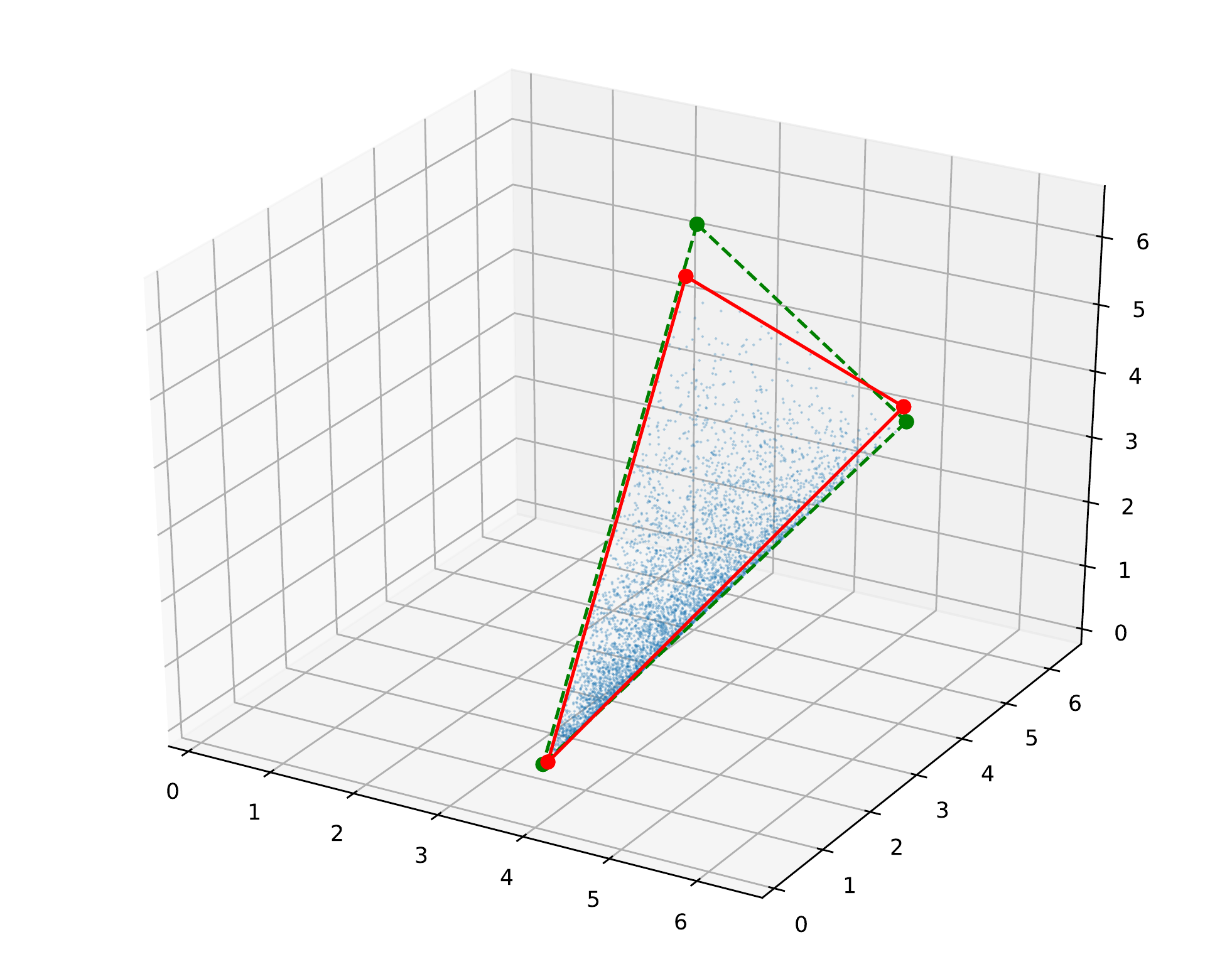}
  \caption{HMC}
  \label{fig:stan-hmc_alpha}
\end{subfigure}
\begin{subfigure}{.5\textwidth}
  \centering
  \includegraphics[width=\linewidth]{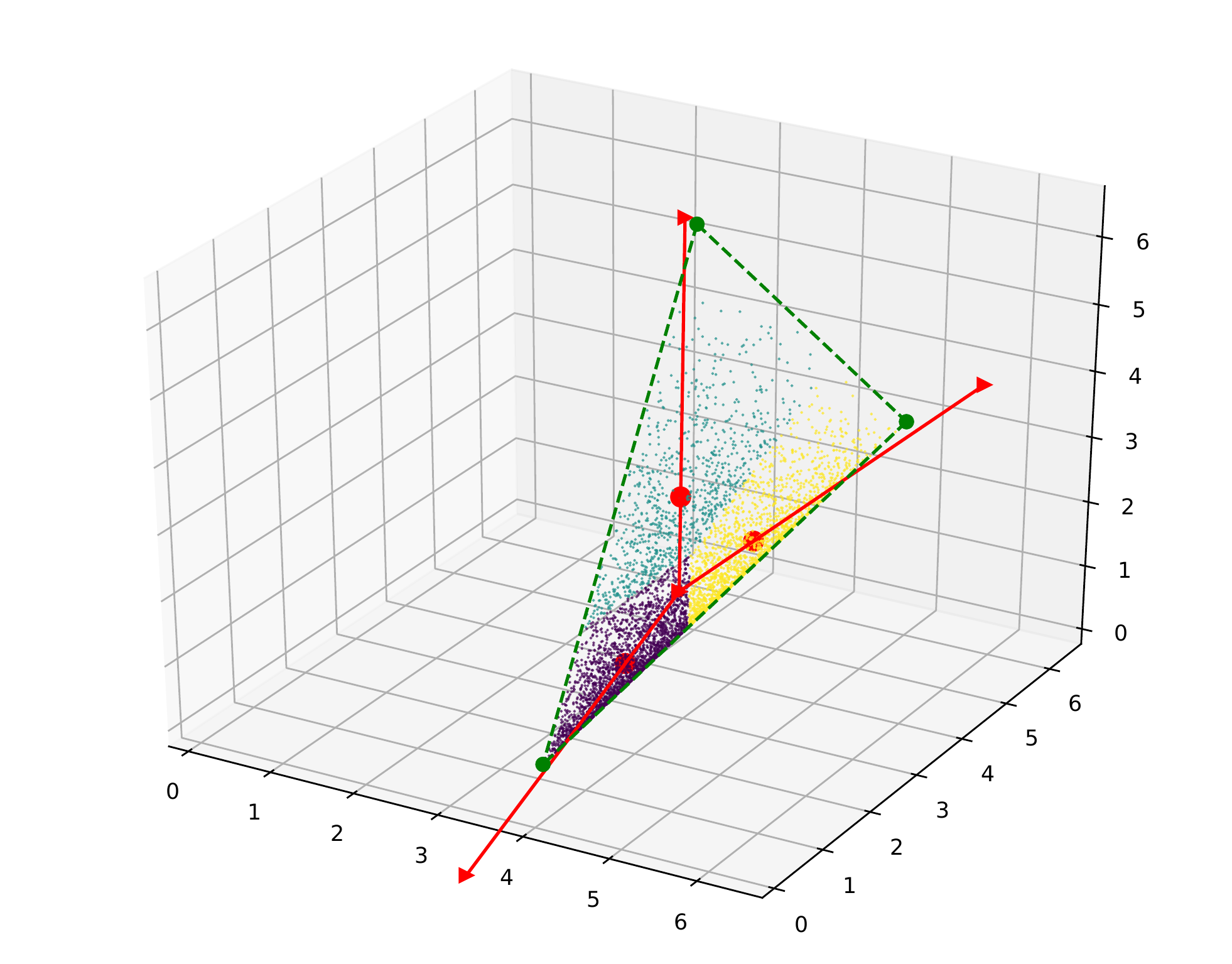}
  \caption{VLAD}
  \label{fig:gam_alpha}
\end{subfigure}
\caption{Asymmetric Dirichlet toy simplex learning: $n=5000,D=3,K=3,\alpha=(0.5,1.5,2.5)$}
\label{fig:triangle_alpha}
\vskip -0.2in
\end{figure*}

\end{document}